\documentclass{article}
\usepackage{iclr2027_conference,times}
\iclrfinalcopy
\usepackage{amsmath,amsfonts,bm}

\def\eqref#1{equation~\ref{#1}}

\def\1{\bm{1}}

\DeclareMathAlphabet{\mathsfit}{\encodingdefault}{\sfdefault}{m}{sl}
\SetMathAlphabet{\mathsfit}{bold}{\encodingdefault}{\sfdefault}{bx}{n}

\renewcommand{\eqref}[1]{\textup{(\ref{#1})}}

\usepackage{amsmath,amssymb,amsthm,mathtools,booktabs}
\usepackage{hyperref}
\hypersetup{hidelinks}
\usepackage{url}
\usepackage{thmtools,thm-restate}
\usepackage{tikz}
\usepackage{float}
\usepackage{subcaption}
\usepackage{graphicx,xcolor}
\usepackage{microtype}

\newcommand{\LN}{\mathrm{LN}}

\newcommand{\one}{\mathbf{1}}

\newcommand{\Tset}{\mathcal{T}_L}
\newcommand{\yhat}{\hat{y}}
\newcommand{\Bmin}{B_{\min}}

\definecolor{fcoral}{HTML}{C24E28}
\definecolor{fteal}{HTML}{15805E}

\theoremstyle{plain}
\newtheorem{theorem}{Theorem}
\newtheorem{lemma}{Lemma}

\newtheorem{proposition}{Proposition}

\title{Counting with Attention and LayerNorm:\\
Norm Scaling and Representation Geometry}
\author{Ayyub Shaffy$^{*}$\quad Manoj Dhakal$^{*}$\quad Safal Shrestha\quad Keith W. Ross \\
New York University Abu Dhabi \\
\texttt{\{am12321,md5121,ss13750,keithwross\}@nyu.edu}
}
\begin{document}

\maketitle
{\let\thefootnote\relax\footnotetext{$^{*}$Equal contribution.}}
\lhead{Preprint. Under review.}

\begin{abstract}
Counting is a simple task on which transformers can nevertheless fail at long sequence lengths. We study the cost of counting in a single attention layer, up to a maximum length $L$, focusing on how softmax attention and layer normalization shape the representation from which the count is recovered.
We mainly focus on regression read-out but also consider classification read-out.
For regression, we show unnormalized linear attention  counts exactly at parameter norms independent of $L$. For softmax attention, regression cannot count exactly with finite parameters; it recovers the count after rounding, but only when the vector parameter norms grow with $L$. We measure this cost by $\Bmin$, the smallest bound that can be placed on every trainable vector
to ensure exact counting after rounding, and derive tight rates. At fixed model dimension, $\Bmin$ grows as $L$ without normalization, as $L^{3/2}$ under bare post-normalization, and as $L^{3/4}$ with learned post-normalization. The last rate is unexpected: post-normalization discards the magnitude that carries the count and might be expected to make counting harder, yet with a learnable gain it counts at a smaller norm than softmax without normalization. Bare pre-normalization is different: it bounds the attention scores, so beyond a finite length no parameter norm suffices. A classification read-out keeps every norm bounded under every placement, but a dense head stores one output vector per count and costs more in storage. Experiments on the post-normalized model reproduce this split between the two read-outs. Together, our theoretical and empirical findings characterize how self-attention, layer-normalization placement, and read-out choice shape the representation geometry and parameter cost of counting in the minimal single-head architecture studied here.

\end{abstract}

\section{Introduction}
Although counting is conceptually simple, modern transformer models often fail to count reliably over long sequence lengths $L$. Recent empirical and theoretical investigations reach conflicting conclusions regarding the root cause of this failure. Some studies identify Layer Normalization (LN) and the softmax function as important sources of difficulty due to scale erasure \citep{ouellette2023counting}, others focus on embedding crowding when vocabulary size approaches model dimension \citep{yehudai2024count}, while others emphasize positional encoding constraints \citep{cho2024position, mcleish2024abacus}. Existing constructions do not by themselves determine how the placement of layer normalization changes the parameter norms needed as the counting range grows.

In this paper, we isolate the mathematical properties of counting within a single attention sublayer by analyzing a basic benchmark task: counting occurrences of spaces in a sentence, for sentences of maximum length $L$ tokens. We evaluate three normalization placements---no LN, pre-LN, and post-LN---with bare and learnable normalization. We primarily address regression read-out regimes but also consider classification. For regression, with one exception, we establish that model parameters can be configured such that a single-layer attention mechanism recovers the count up to sequence length $L$ once the prediction is rounded. The exception is bare pre-LN, whose attention scores are bounded, so that beyond a finite length no parameter norm suffices. For the other softmax regression models, correct counting requires parameter norms that grow with $L$, at a rate that depends on where the normalization sits. The geometry of the boundary representation explains why: each placement compresses the counts in its own way, and the read-out must undo that compression.

Our main contributions are as follows:

\begin{itemize}
    \item We prove that a single attention sublayer---without feed-forward sublayers or positional encodings---can count up to any fixed maximum length $L$ under regression followed by rounding, except for the finite cap of bare pre-LN softmax regression. Unnormalized linear attention counts exactly at a length-independent norm budget, whereas the softmax regression models require growing norms.

    \item We derive tight bounds for the minimum norm budget $\Bmin(L)$ under regression followed by rounding:
$\Theta(\max\{L,\,d^{1/4}\sqrt{\log L}\})$ for unnormalized softmax, $\Theta(L^{3/2}/\sqrt d)$ for bare post-LN, $\Theta(L^{3/4}/d^{1/4})$ for post-LN with a learned gain and bias, and $\Theta(L/d^{1/4})$ for learnable pre-LN, where $d$ is the model dimension. Surprisingly, at fixed dimension post-LN with a learned gain has the slowest budget growth of all the softmax variants, even though normalization discards the magnitude that carries the count.

    \item We show that the rates, including this surprising one, follow from the geometry of the representation. Softmax attention places the counts on a bounded segment, and post-normalization places them on a circular arc. The arc stays nearly straight for longer, so a read-out does not have to stretch the small gaps between counts as much.

    \item We show that bare pre-LN is the one placement that larger parameters cannot rescue: normalizing the embeddings bounds the attention scores and caps the counting range at $L_{\mathrm{cap}}(d)=\Theta(e^{\sqrt d})$ under regression.

    \item Classification removes the cap and counts at bounded norms under every placement. Its dense head uses $\Theta(L)$ times the storage of a regression read-out at a common precision, while the minimum winning margins in our constructions shrink with $L$.

    \item We use L-BFGS to train both read-outs under post-normalization from random initialization: the regression norm grows with $L$, on or above the predicted band, while the classifier keeps bounded norms and collapsing margins. In the unnormalized model trained with Adam, the saturation scale $\gamma$, not the parameter norm, is the quantity that lags during training.
\end{itemize}
Overall, our analysis uncovers how representation geometry and normalization placement fundamentally shape the parameter efficiency and feasibility of counting with attention.
\section{Related work}
\label{sec:related}

Counting can be implemented sequentially, with one increment per token. In a literal running-counter implementation, a token cannot read a count its predecessor is still computing in the same layer, so each handoff costs a layer; this limits that implementation rather than all transformer counting algorithms (Appendix~\ref{app:necessity}). Other work studies transformer expressivity, parallel computation, and the additional computation available through chain of thought \citep{merrill2023parallelism,merrill2022saturated,merrill2024cot,liu2023shortcuts,hahn2020theoretical}. \citet{hasani2026mechanistic} find that language models accumulate counts across layers, so that precision degrades on large counts, and recover accuracy by splitting the input into smaller parts that are counted separately. A separate line asks which counting-like properties transformers can recognize \citep{chiang2022overcoming,weiss2021thinking,bhattamishra2020ability,saelzer2026counting}. Among these, \citet{chiang2022overcoming} construct transformers that decide \textsc{Parity}, whether a bit string has an odd number of 1s, correctly at every length, but whose output probabilities approach one half as inputs grow longer. Adding layer normalization rescales the representation, so the output stays confident at every length.

\citet{behrens2025counting} study the histogram task, where each token reports how many times it appears in the sequence, and give explicit weight constructions for a family of one-layer models. With a boundary token, the model can compare tokens pairwise and a single hidden neuron suffices; without one, their softmax constructions keep an inventory of the alphabet, with one hidden neuron per symbol. Their analysis does not consider layer normalization. With no feed-forward layer and a linear read-out, we show that softmax attention cannot count without a parameter norm that grows with the length, and we give the rate at each placement of the normalization. \citet{yehudai2024count} bound counting by the embedding dimension $d$ against the vocabulary size $m$: their histogram construction needs enough dimension to represent the vocabulary, while their CountAttend construction needs a feed-forward network whose Lipschitz constant grows with the counting range. Their general lower bound also implies that the model's Lipschitz constants must grow with the context length. We instead fix a two-symbol content alphabet, so that the vocabulary is not the bottleneck, and study the norm cost of a linear read-out, with normalization included explicitly. Empirically, \citet{ouellette2023counting} identify layer normalization and softmax averaging as obstacles to count generalization, and show that a small encoder with neither one counts correctly on input sizes beyond those it was trained on.

Counting is one case of length generalization; arithmetic is another widely studied instance \citep{lee2024teaching, zhou2024length}. \citet{cho2024position} and \citet{mcleish2024abacus} write task structure into the positions to improve length generalization in arithmetic tasks.

Earlier work establishes whether transformers can count, at what depth, and with which constructions. It identifies layer normalization and softmax as empirical obstacles \citep{ouellette2023counting}, but also shows that layer normalization can help \citep{chiang2022overcoming}. \citet{yehudai2024count} show that parameters must grow with the context length, but do not ask how the placement of layer normalization changes that growth. We answer this question for each placement, no LN, pre-LN and post-LN, with bare and learnable normalization, and for both regression and classification read-outs.

\section{Problem setup}
\label{sec:setup}

We study counting on a single-layer attention model, reduced to the smallest architecture on which the phenomena of interest can be established analytically. The task is to predict the number of occurrences of a designated token in an input that also contains distractor tokens. We use space counting as the concrete instance, treating each character as a token and appending a boundary symbol, and represent all letters by a single symbol with a shared embedding.

Formally, fix an alphabet of two content symbols, a space $s$ and a letter $\ell$, together with a distinguished boundary symbol $\#$. An input is a sequence of content symbols terminated by the boundary symbol. We write $n$ for the number of spaces and $m$ for the number of letters, and require $n+m\le L$, where $L$ is a fixed maximum length. Because the model below is permutation-invariant across content positions, its output depends on the input only through the pair $(n,m)$. We therefore identify an input with $n$ spaces and $m$ letters as a \emph{type} $(n,m)$ input, where $(n,m)\in\Tset:=\{(n,m)\in\mathbb N_0^2 : n+m\le L\}$. The goal is to predict the space count $n$.

\paragraph{Model.}
We use a single layer with one single-head attention sublayer, with no feed-forward sublayer, positional encoding, residual connection, or separate read-out bias. We denote by $d$ the model dimension and take $d\ge3$, $L\ge4$ in the main theorems. The space, letter, and boundary symbols receive embeddings $e_s,e_\ell,e_\#\in\mathbb R^d$, so an input of type $(n,m)$ presents $n$ copies of $e_s$, $m$ copies of $e_\ell$, and one copy of $e_\#$ at the end. The head computes scores of the boundary query against each key,
\begin{equation}
c_s = \frac{(W_Qe_\#)^\top(W_Ke_s)}{\sqrt d}, \qquad
c_\ell = \frac{(W_Qe_\#)^\top(W_Ke_\ell)}{\sqrt d}, \qquad
c_\# = \frac{(W_Qe_\#)^\top(W_Ke_\#)}{\sqrt d}.
\label{eq:scores}
\end{equation}
We fix $W_Q=W_K=W_V=I$ in all main results, so that $c_s=e_s^\top e_\#/\sqrt d$, $c_\ell=e_\ell^\top e_\#/\sqrt d$, and $c_\#=\|e_\#\|^2/\sqrt d$. Learned query and key projections appear only in the extension of Appendix~\ref{app:pre-extensions}. We consider two attention rules, treated as separate architectural choices. Under softmax attention, the weight on a token is proportional to the exponential of its score, so on type $(n,m)$ the boundary representation is
\begin{equation}
h_{\mathrm{softmax}}(n,m) = \frac{n\,e^{c_s}e_s + m\,e^{c_\ell}e_\ell + e^{c_\#}e_\#}
{n\,e^{c_s} + m\,e^{c_\ell} + e^{c_\#}}.
\label{eq:hsoftmax}
\end{equation}
Under linear attention, by which we mean unnormalized dot-product attention rather than a kernelized rule, the weight on a token is its score:
\begin{equation}
h_{\mathrm{lin}}(n,m) = n\,c_s e_s + m\,c_\ell e_\ell + c_\# e_\#.
\label{eq:hlin}
\end{equation}

The main text studies the \textbf{regression} read-out, a scalar prediction $\yhat(n,m)=w^\top h(n,m)$ with $w\in\mathbb R^d$, where $h$ is $h_{\mathrm{softmax}}$ or $h_{\mathrm{lin}}$, possibly after normalization. Section~\ref{sec:classification} studies the \textbf{classification} read-out, with class vectors $\{w_c\}_{c=0}^L$, biases $b_c$, and prediction $\arg\max_c(w_c^\top h(n,m)+b_c)$. We collect the trainable quantities into $\theta=(e_s,e_\ell,e_\#,w)$, augmented by the normalization parameters when present.

\paragraph{Error and budget.}
We measure a model by its worst-case absolute error over all admissible inputs, and a model class by the best worst-case error over all parameters:
\begin{equation}
E(\theta;L) = \max_{(n,m)\in\Tset}\bigl|\yhat_\theta(n,m) - n\bigr|, \qquad
E^\star(L) = \inf_\theta E(\theta;L).
\label{eq:Estar}
\end{equation}
A model class \emph{recovers the count exactly} if $E^\star(L)=0$ and the infimum is attained by finite parameters, and \emph{arbitrarily closely} if $E^\star(L)=0$ but the infimum need not be attained. It has a \emph{finite rounding cap} if, beyond some length, error below $1/2$ is impossible at every norm budget.

Parameter norms are constrained in practice, by regularization and by numerical range, so we also constrain them. For $B>0$, let $\Theta_B=\{\theta:\max\{\|e_s\|,\|e_\ell\|,\|e_\#\|,\|w\|\}\le B\}$, with each learned gain and normalization bias also bounded by $B$ when present, and write $E^\star(L;B)=\inf_{\theta\in\Theta_B}E(\theta;L)$. Error strictly below $1/2$ guarantees correct nearest-integer rounding on every input, so we define
\begin{equation}
\Bmin(L,d):=\inf\{B>0:E^\star(L,d;B)<1/2\},
\label{eq:rounding-budget}
\end{equation}
with the infimum of the empty set taken to be $+\infty$. We suppress $d$ when it is fixed. These are expressivity bounds: they say which parameters exist, not that training finds them. Read the other way, $L_{\max}(B)=\sup\{L:E^\star(L;B)<\tfrac12\}$ is how far an architecture can count within budget $B$, and a finite cap means  $\sup\{L : E^\star(L,d) < 1/2\} < \infty$.

\paragraph{Why a peak-norm budget.}
We bound each trainable vector separately, for two reasons. First, this is how norm constraints are commonly imposed in training. Max-norm constraints bound each incoming weight vector independently and were used alongside dropout in early work \citep{hinton2012improving,srivastava2014dropout}. A small $\Bmin$ means that counting remains possible under a tight constraint of this kind. Second, numerical formats limit the largest value a model can store. A single large entry can exceed the range of a low-precision format, and under shared-scale quantization it also reduces the resolution left for the other entries in the same tensor \citep{dettmers2022llmint8}. Since the largest entry and the largest vector norm differ by at most a factor $\sqrt{d}$, a smaller $\Bmin$ means smaller values to represent. Because vectors are bounded separately, $\Bmin$ depends on how the model is parameterized: a coefficient written as a product of two vectors can be supplied by both together. Section~\ref{sec:postln} returns to this, since it explains the effect of a learned gain.

\section{Counting with Regression}
\subsection{Counting without layer normalization}
\label{sec:nonorm}

We first analyze the model without normalization, which isolates the effect of the attention rule. Write $\alpha=w^\top e_s$ for the read-out of the space embedding, and write $\gamma(\theta)=e^{c_\#-c_s}$ for the \emph{saturation scale}, the attention weight on the boundary token relative to one space token. Let $\Theta_\gamma=\{\theta\mid\gamma(\theta)=\gamma\}$ and $E^\star(\gamma,L,d)=\inf_{\theta\in\Theta_\gamma}E(\theta;L)$.
Consecutive counts sit only a small distance apart in the representation, and the read-out must turn that distance into a change of one in the prediction. We call the factor by which it stretches the representation the \emph{magnification}. In our constructions it is $\gamma$ without normalization and $1/\omega$ after post-normalization (below and Section~\ref{sec:postln}).
For all theorems, proofs are provided in the appendix.

\begin{theorem}[Regression without layer normalization]
\label{thm:nonorm}
\leavevmode
\begin{enumerate}
\item[(a)] \textbf{Linear attention.} There exist a read-out $w$ and linearly independent embeddings $e_s,e_\ell,e_\#$, independent of $L$, with $\yhat(n,m)=n$ for all $(n,m)\in\Tset$. Dimension $d=3$ suffices, and $\Bmin^{\mathrm{lin}}(L,d)=\Theta(d^{1/8})$, which is $\Theta(1)$ in $L$ at fixed dimension.

\item[(b)] \textbf{Softmax attention.} For $\gamma\ge1$, the optimal worst-case error is $E^\star(\gamma,L,d)=\Theta(L^2/(L+\gamma))$. Exact predictions are impossible at finite parameters, but the error tends to zero as $\gamma\to\infty$, so $E^\star(L,d)=0$. With rounding, exact predictions are achievable, with $\Bmin^{\mathrm{soft}}(L,d)=\Theta\bigl(\max\{L,\,d^{1/4}\sqrt{\log L}\}\bigr)$, which is linear in $L$ at fixed dimension.
\end{enumerate}
The implicit constants are independent of $L$ and $d$, and in the error bound also of $\gamma$.
\end{theorem}

\begin{figure}[t]
\centering
\resizebox{\linewidth}{!}{%
\begin{tikzpicture}[scale=1.0,>=stealth,font=\small]

\begin{scope}[xshift=0cm]
  \node[font=\scriptsize\bfseries,anchor=west] at (-0.05,5.02) {(a) Linear attention, no LN};

  \draw[gray!40,dashed,line width=0.3pt] (0.30,1.05) -- (1.15,1.377);
  \fill[gray!60] (0.30,1.05) circle (1.0pt);
  \node[gray!70,font=\scriptsize,anchor=north east,inner sep=1.5pt] at (0.26,1.00) {$0$};

  \foreach \n in {0,...,24}{
    \fill[fteal] ({0.95+0.037*\n},{1.30+0.111*\n}) circle (1.15pt);}
  \draw[fteal,opacity=0.40,line width=0.7pt,->] (1.88,4.03) -- (1.96,4.27);
  \fill[fteal] (0.95,1.30) circle (2.2pt);

  \node[font=\scriptsize,anchor=north,inner sep=3.5pt] at (1.14,1.20) {$n{=}0$: $c_\#e_\#$};
  \node[font=\scriptsize,anchor=east,inner sep=3.5pt] at (1.76,3.97) {$n{=}24$};
\end{scope}

\begin{scope}[xshift=4.15cm]
  \node[font=\scriptsize\bfseries,anchor=west] at (-0.05,5.02) {(b) Softmax attention, no LN};

  \draw[gray!40,dashed,line width=0.3pt] (0.30,1.05) -- (2.75,1.377);
  \draw[gray!40,dashed,line width=0.3pt] (0.30,1.05) -- (1.16,4.02);
  \fill[gray!60] (0.30,1.05) circle (1.0pt);
  \node[gray!70,font=\scriptsize,anchor=north east,inner sep=1.5pt] at (0.26,1.00) {$0$};

  \draw[fcoral,opacity=0.35,line width=0.7pt] (2.55,1.35) -- (1.10,3.90);
  \foreach \n in {0,...,24}{
    \pgfmathsetmacro{\lam}{\n/(\n+4)}
    \fill[fcoral] ({2.55-1.45*\lam},{1.35+2.55*\lam}) circle (1.15pt);}
  \fill[fcoral] (2.55,1.35) circle (2.2pt);
  \draw[fcoral,line width=0.7pt,fill=white] (1.10,3.90) circle (2.2pt);

  \draw[gray!65,line width=0.5pt,->] (2.639,1.362) -- (2.946,1.403);
  \draw[gray!65,line width=0.5pt,->] (1.124,3.987) -- (1.176,4.170);

  \node[font=\scriptsize,anchor=north,inner sep=3.5pt] at (2.55,1.26) {$n{=}0$: $e_\#$};
  \node[font=\scriptsize,anchor=east,inner sep=3.5pt] at (1.02,3.90) {$e_s$};
  \draw[gray!55,line width=0.3pt] (1.36,3.51) -- (1.73,3.41);
  \node[font=\scriptsize,anchor=west,inner sep=2pt] at (1.73,3.4) {$n{=}24$};
\end{scope}

\begin{scope}[xshift=8.30cm]
  \node[font=\scriptsize\bfseries,anchor=west] at (-0.05,5.02) {(c) Post-normalization};

  \draw[gray!55,->,line width=0.3pt] (0.40,1.15) -- (3.60,1.15)
    node[font=\scriptsize,right,inner sep=1.5pt] {$u$};
  \draw[gray!55,->,line width=0.3pt] (0.40,1.15) -- (0.40,3.90)
    node[font=\scriptsize,anchor=south east,inner sep=1.5pt] {$v$};

  \foreach \n in {0,1,2,3,4}{
    \draw[gray!35,line width=0.25pt] (0.40,1.15) -- ({3.20},{1.15+0.70*\n});}

  \draw[fcoral,opacity=0.35,line width=0.7pt] ([shift={(0.40,1.15)}]0:2.35) arc (0:81:2.35);
  \foreach \n in {0,...,24}{
    \pgfmathsetmacro{\ang}{atan(\n/4)}
    \fill[fcoral] ({0.40+2.35*cos(\ang)},{1.15+2.35*sin(\ang)}) circle (1.15pt);}

  \draw[fteal,opacity=0.35,line width=0.7pt] (3.20,1.15) -- (3.20,3.95);
  \foreach \n in {0,...,4}{ \fill[fteal] (3.20,{1.15+0.70*\n}) circle (1.15pt);}
  \draw[fteal,opacity=0.40,line width=0.7pt,->] (3.20,3.99) -- (3.20,4.24);

  \node[font=\scriptsize,fteal,anchor=east,inner sep=2pt] at (4.35,3.06) {$A+nC$};
  \node[font=\scriptsize,fcoral,anchor=west,inner sep=2pt] at (0.92,3.86) {$\mathrm{LN}\bigl(h(n)\bigr)$};
  \node[font=\scriptsize,anchor=west,inner sep=2pt] at (0.45,3.2) {$n{=}24$};
  \node[font=\scriptsize,anchor=north,inner sep=3pt] at (2.75,1.2) {$n{=}0$};
\end{scope}

\begin{scope}[xshift=12.45cm]
  \node[font=\scriptsize\bfseries,anchor=west] at (-0.05,5.02) {(d) Bare pre-LN, softmax};

  \draw[gray!50,dashed,line width=0.35pt] (1.80,2.45) circle (1.42);
  \fill[gray!60] (1.80,2.45) circle (1.0pt);
  \node[gray!70,font=\scriptsize,anchor=north east,inner sep=1.5pt] at (1.75,2.40) {$0$};
  \draw[gray!60,line width=0.35pt] (1.80,2.45) -- (3.087,1.850);
  \draw[gray!60,line width=0.35pt] (1.80,2.45) -- (0.887,3.538);
  \node[gray!75,font=\scriptsize,anchor=north,inner sep=2.5pt] at (2.47,2.11) {$\sqrt d$};
  \node[gray!75,font=\scriptsize,anchor=east,inner sep=2.5pt]  at (1.29,2.96) {$\sqrt d$};

  \pgfmathsetmacro{\gammapre}{0.5*exp(sqrt(3)*(1-cos(155)))}

  \draw[fcoral,opacity=0.35,line width=0.7pt] (3.087,1.850) -- (0.887,3.538);
  \foreach \n in {0,...,24}{
    \pgfmathsetmacro{\lam}{\n/(\n+\gammapre)}
    \fill[fcoral] ({3.087-2.200*\lam},{1.850+1.688*\lam}) circle (1.15pt);}
  \fill[fcoral] (3.087,1.850) circle (2.2pt);
  \draw[fcoral,line width=0.7pt,fill=white] (0.887,3.538) circle (2.2pt);

  \node[font=\scriptsize,anchor=north,inner sep=3.5pt] at (3.2,1.80) {$n{=}0$: $z_\#$};
  \node[font=\scriptsize,anchor=east,inner sep=3.5pt] at (0.82,3.50) {$z_s$};

  \pgfmathsetmacro{\lamlast}{24/(24+\gammapre)}
  \draw[gray!55,line width=0.3pt]
    ({3.087-2.200*\lamlast},{1.850+1.688*\lamlast}) -- ++(0.57,0.25)
    node[text=black, font=\scriptsize,anchor=west,inner sep=2pt] {$n{=}24$};
\end{scope}

\node[font=\scriptsize,anchor=north west,text width=3.1cm,align=left] at (0,0.65)
  {Every space adds the same vector $c_se_s$};
\node[font=\scriptsize,anchor=north west,text width=3.0cm,align=left] at (4.15,0.65)
  {Steps shrink as the points approach $e_s$};
\node[font=\scriptsize,anchor=north west,text width=3.75cm,align=left] at (8.30,0.70)
  {Equal steps in $A + nC$ become unequal angular steps};
\node[font=\scriptsize,anchor=north west,text width=3.75cm,align=left] at (12.45,0.70)
  {Fixed endpoint norms bound the saturation scale};

\end{tikzpicture}%
}
\caption{Boundary representations $h(n,0)$ on space-only inputs. (a) Linear attention gives equally spaced points on a ray. (b) Unnormalized softmax compresses them along a segment. (c) Post-normalization maps the centered numerator onto an arc. (d) Bare pre-normalization constrains the segment endpoints to a sphere. Open circles mark limiting endpoints.}
\label{fig:flagship}
\end{figure}
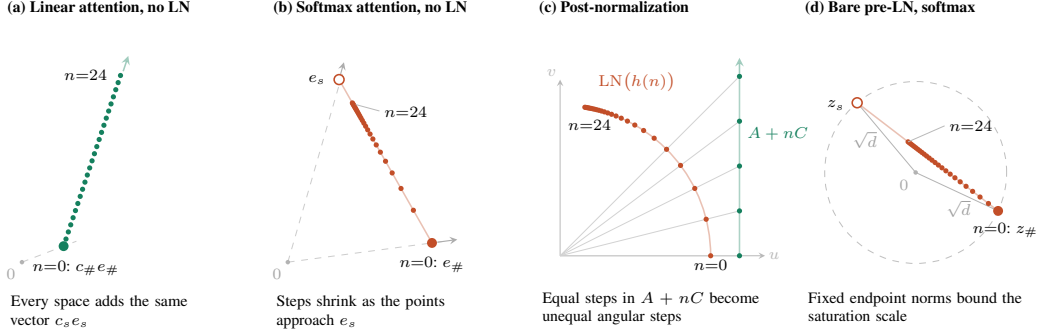

As shown in the proof, linear attention counts because $h_{\mathrm{lin}}$ is affine in $(n,m)$: every space adds the same vector $c_se_s$, the representation moves along a ray with equal steps (Figure~\ref{fig:flagship}(a)), and a read-out that isolates that direction reads $n$ off exactly, with parameters that do not depend on $L$.

Softmax attention divides by the total attention weight with $n e^{c_s}+m e^{c_\ell}+e^{c_\#}$, which grows with every token. The representation $h(n,m)$ is a convex combination of the three embeddings $e_s$, $e_\ell$, $e_\#$, and stays inside the triangle they span no matter how long the input. On space-only inputs it moves along the segment from $e_\#$ toward $e_s$, at coordinate $\lambda_n=n/(n+\gamma)$ (Figure~\ref{fig:flagship}(b)). The steps shrink as $n$ grows, and an exact linear read-out would need them equal, so exact counting fails at every finite $\gamma$. Rounding only needs error below $1/2$, which the construction achieves by keeping all $L$ counts near the start of the segment, where $\lambda_n\approx n/\gamma$ and the steps are nearly equal; this requires $\gamma$ of order $L^2$. The read-out must then turn a displacement of order $1/\gamma$ per count into a change of one, so the magnification $\alpha=w^\top e_s$ must be of order $\gamma\approx L^2$. Since $\alpha\le\|w\|\|e_s\|$, the two vectors can supply it together with each only needing a norm of order $L$, which is the $\Theta(L)$ in part (b). The second term in the budget is the cost of the score gap $c_\#-c_s=\log\gamma$, which enters through an exponential and needs norm only of order $d^{1/4}\sqrt{\log L}$. Appendices~\ref{app:softmax-lemma} and~\ref{app:precision} give the lower bound and the exact spacing identity without any orthogonality choice.

\subsection{Post-normalization: the count becomes an angle}
\label{sec:postln}

We now add layer normalization after attention, so the boundary representation is formed by attention and then normalized before the read-out. Bare normalization, denoted $\LN_0$, subtracts the coordinate mean and rescales to a fixed length,
\begin{equation}
\LN_0(v):=\sqrt d\,\frac{Pv}{\|Pv\|_2},\qquad Pv:=v-\frac{\one^\top v}{d}\one,
\label{eq:s5-ln}
\end{equation}
so it keeps the direction of the centered vector and discards its magnitude. Learnable normalization is $\LN_{g,\beta}(v)=g\odot\LN_0(v)+\beta$ with $g,\beta\in\mathbb R^d$. We use stability constant $\varepsilon=0$, require every normalized vector to be well defined, and work in exact arithmetic.

\begin{theorem}[Regression with post-normalization]
\label{thm:post}
For either softmax or linear attention, bare post-normalization and post-normalization with learned parameters satisfy $E^\star(L)=0$, but no finite parameters achieve zero error.
\begin{enumerate}
\item[(a)] \textbf{Bare post-normalization.} Bounding each of $e_s,e_\ell,e_\#,w$ by $B$,
\begin{equation}
E^\star(L,d;B)=\Theta\!\left(\min\left\{L,\frac{L^3}{dB^2}\right\}\right),\qquad
\Bmin^{\mathrm{post}}(L)=\Theta\!\left(\frac{L^{3/2}}{\sqrt d}\right).
\label{eq:s5-post-budget}
\end{equation}
\item[(b)] \textbf{Learned gain and bias.}
Bounding each of $e_s,e_\ell,e_\#,g,\beta,w$ by $B$,
\begin{equation}
E^\star(L,d;B)=\Theta\!\left(
\min\left\{L,\frac{L^3}{dB^4}\right\}\right),\qquad
\Bmin^{\mathrm{post}}(L)=
\Theta\!\left(\frac{L^{3/4}}{d^{1/4}}\right).
\label{eq:s5-post-learned-budget}
\end{equation}
\end{enumerate}
The constants are independent of $L$, $d$, and $B$.
\end{theorem}

\paragraph{The arc.}
Set the letter embedding so that $Pe_\ell=0$.
Under softmax attention the denominator of \eqref{eq:hsoftmax} divides both the centered vector and its norm, so it cancels inside $\LN_0$. Both attention rules then give, on every input,
\begin{equation}
\LN_0(h(n,m))=\sqrt d\,\frac{A+nC}{\|A+nC\|_2},\qquad A=Pe_\#,\quad C=\rho Pe_s,
\label{eq:s5-post-ray}
\end{equation}
where $\rho=e^{c_s-c_\#}$ for softmax and $\rho=c_s/c_\#$ for linear attention. The vector $A+nC$ advances by the same increment for every space, as in linear attention. Normalization then projects these evenly spaced points radially onto a sphere of radius $\sqrt d$. If $A\perp C$, with $u=A/\|A\|$, $v=C/\|C\|$ and $\omega=\|C\|/\|A\|$,
\begin{equation}
\LN_0(h(n,m))=\sqrt d\bigl(\cos\theta_n\,u+\sin\theta_n\,v\bigr),\qquad\theta_n=\arctan(\omega n).
\label{eq:s5-angle}
\end{equation}
The count is now an angle on a circular arc (Figure~\ref{fig:flagship}(c)), and the angular steps shrink as $n$ grows. A linear read-out measures a projection of this arc. Since $v^\top\LN_0(h)=\sqrt d\sin\theta_n\approx\sqrt d\,\omega n$ for small $\omega n$, the choice $w=v/(\omega\sqrt d)$ gives $\yhat(n,m)=\sin\theta_n/\omega=n/\sqrt{1+\omega^2n^2}$, with $0\le n-\yhat(n,m)\le\omega^2n^3/2$. A small $\omega$ keeps all counts near the start of the arc, where the projection is nearly linear in $n$, but it also shrinks the angular steps the read-out must amplify, since $\|w\|=1/(\omega\sqrt d)$. Setting $\omega=L^{-3/2}$ makes the error strictly below $1/2$ at norm $L^{3/2}/\sqrt d$, and part (a) states this rate is optimal, with no orthogonality assumption.

\paragraph{Why the arc needs less magnification than the segment.}
Both constructions start from evenly spaced points and divide them by a quantity that grows with the count. On space-only inputs, softmax divides by the total attention weight, which is proportional to $\gamma+n=\gamma(1+n/\gamma)$. Post-normalization divides by the length of $A+nC$, which for $A\perp C$ is $\|A\|\sqrt{1+(\omega n)^2}$ with $\omega=\|C\|/\|A\|$. In both constructions the ratio of the space contribution to the boundary contribution, $n/\gamma$ or $\omega n$, stays small for every count up to $L$. What differs is how it enters the divisor: linearly for softmax, squared for post-normalization. A ratio of $0.1$, for example, enlarges the softmax divisor by $10\%$ but the post-normalization divisor by only $0.5\%$. The arc therefore stays nearly straight for longer, and the read-out can place the largest count further along it before its error reaches $1/2$, at $\theta_L$ of order $L^{-1/2}$ rather than $\lambda_L$ of order $L^{-1}$. Turning that count into the output $L$ then needs magnification $1/\omega=L^{3/2}$ on the arc but $\gamma=L^2$ on the segment. In summary, post-normalization needs less magnification because it replaces the softmax's compression with a gentler one, rather than adding a second independent constraint on top of it.
\paragraph{Why the learned gain lowers the budget.}
Without normalization the magnification $L^2$ is a product $w^\top e_s$, so two vectors share it and each needs norm of order $L$. Under bare post-normalization the representation has fixed length $\sqrt d$, so $w$ must supply all of $L^{3/2}$ alone; this is why bare post-LN, despite needing less magnification, has the larger budget. A learned gain restores the sharing. With $\beta=0$ the prediction is $(g\odot w)^\top\LN_0(h)$, so $g$ and $w$ act through their product, and bounds $\|g\|,\|w\|\le B$ allow an effective read-out of norm order $B^2$. The same arc construction then rounds correctly at $B$ of order $L^{3/4}/d^{1/4}$, which is part (b). At fixed dimension this is smaller than the $\Theta(L)$ of unnormalized softmax: the sharing is the same, but there is less magnification to share.

\subsection{Pre-normalization: a finite cap}
\label{sec:ln:pre}

Under pre-normalization, each embedding is normalized before attention, so the boundary representation is built from the normalized values $z_s,z_\ell,z_\#$, each of length $\sqrt d$ under bare normalization. The geometry of Section~\ref{sec:nonorm} is unchanged: on space-only inputs $h(n,0)=z_\#+\lambda_n(z_s-z_\#)$ with $\lambda_n=n/(n+\gamma)$, the same segment (Figure~\ref{fig:flagship}(d)). What changes is what the parameters can reach.

\begin{theorem}[Bare pre-normalization]
\label{thm:bare}
\label{thm:pre}
For softmax attention with bare pre-normalization, the saturation scale satisfies
$1 \le \gamma \le e^{2\sqrt{d}}$, and the optimal error, with
unrestricted parameter norms, is
\[
  E^\star(L,d) = \Theta\!\left(\frac{L^2}{L + e^{2\sqrt{d}}}\right).
\]
Hence error below $\tfrac12$ is achievable only for
$L \le L_{\mathrm{cap}}(d) := \sup\{L : E^\star(L,d) < \tfrac12\} = \Theta\big(e^{\sqrt{d}}\big)$.
\end{theorem}

Keeping the counts near the start of the segment requires $\gamma$ of order $L^2$, as in Section~\ref{sec:nonorm}. But $\gamma=e^{c_\#-c_s}$, and the scores are bounded by the fixed token lengths, $c_\#=\sqrt d$ and $|c_s|\le\sqrt d$. So $\gamma$ cannot grow with $L$, the counts beyond some length spread onto the range where the segment coordinate is strongly nonlinear in the count. The resulting error cannot be changed by any readout since a larger read-out amplifies the gaps but cannot make them equal.

This is the only configuration we study in which more norm does not help. A learned gain and bias remove the cap, since the gain can lengthen the normalized values, so the scores and $\gamma$ are no longer bounded; the segment geometry of unnormalized softmax returns and $\Bmin^{\mathrm{pre}}(L,d)=\Theta(L/d^{1/4})$, linear in $L$ at fixed dimension (Appendix~\ref{app:learned-pre}). Learned query and key projections also remove it, and linear attention with bare pre-normalization still counts exactly at every length, because pre-normalization leaves its affine structure intact (Appendix~\ref{app:pre-extensions}).

\subsection{The cost of counting, compared}
\label{sec:compare}

\begin{table}[H]
\centering
\caption{Counting costs under a separate norm bound on every trainable vector. Rates are asymptotic in $L$ or $B$ at fixed $d\ge3$; the dependence on $d$ is given with each theorem and in Appendix~\ref{app:consequences}. Softmax rows are ordered by how far they count at a given budget.}
\label{tab:regression}
\footnotesize
\setlength{\tabcolsep}{4pt}
\begin{tabular}{lll}
\toprule
Model & Budget $\Bmin(L)$ & Supported length $L_{\max}(B)$\\
\midrule
Linear, no LN / bare pre-LN & $\Theta(1)$ & $\infty$ once $B$ suffices\\

Softmax, learned post-LN & $\Theta(L^{3/4})$ & $\Theta(B^{4/3})$\\
Softmax, no LN & $\Theta(L)$ & $\Theta(B)$\\
Softmax, learned pre-LN & $\Theta(L)$ & $\Theta(B)$\\
Softmax, bare post-LN & $\Theta(L^{3/2})$ & $\Theta(B^{2/3})$\\
Softmax, bare pre-LN & $\infty$ beyond $L_{\mathrm{cap}}(d)$ & $\le L_{\mathrm{cap}}(d)$\\
\bottomrule
\end{tabular}

\end{table}

Table~\ref{tab:regression} shows the scaling rates for each configuration. The ``Budget'' column gives the norm needed to count to length $L$, and the last column gives how far each architecture can count within a budget $B$. Our surprising main result can be restated using the last column: at the same budget, learned post-LN counts to a length of order $B^{4/3}$, further than any other softmax variant.

The table also shows that layer normalization can either help or hurt counting. Depending on where it is placed and whether its gain is learned, it gives either the best or the worst softmax configuration. Learning the gain helps in both placements. For pre-LN, it removes the finite cap and gives the same rate as unnormalized softmax. For post-LN, it lowers the budget from $L^{3/2}$ to $L^{3/4}$. The placement of bare normalization also matters. Before attention, it caps the counting range. After attention, it allows counting at every length, but at a higher cost than no normalization. Finally, dimension can also substitute for norm. At a fixed budget, the dimension needed to count to length $L$ grows as $L^4$ for learned pre-LN and as $L^3$ for either post-LN variant (Appendix~\ref{app:consequences}).

\section{Counting with a classification head}
\label{sec:classification}

Instead of rounding a scalar, a classifier can keep one score per count and predict $\hat c(n,m)=\arg\max_{0\le c\le L}(w_c^\top h(n,m)+b_c)$, where $h$ is the representation entering the read-out, including post-normalization when present, and $w_c\in\mathbb R^d$, $b_c\in\mathbb R$ are the head parameters.

\begin{theorem}[Classification]
\label{thm:classification}
Softmax attention with no normalization, pre-normalization, and post-normalization can classify every input in $\Tset$ exactly in exact arithmetic, with all parameters bounded independently of $L$ at fixed $d$.
\end{theorem}

\paragraph{Storage.} At first reading of Theorem \ref{thm:classification}, one might think classification is strictly superior to regression for counting with attention. But classification is asymptotically inferior to regression when it comes to storage.
Classification keeps parameter norms bounded, but replaces one regression read-out with one output vector per count. At a common precision, its dense head therefore uses $\Theta(L)$ times as much memory. This separation persists when the models use different precisions: at fixed dimension, dense classification requires $\Omega(L\log L)$ bits, while an unnormalized-softmax regressor can store all its parameters in $O(\log L)$ bits and still count correctly after rounding. Appendix~\ref{app:precision} gives the proof and the precise storage and arithmetic conventions.

Bounded weights suffice for classification because saturation pushes the counts closer together but never out of order. A regression read-out has to recover the distances between counts, and the shrinking distances are what cost it norm; a classifier only has to tell neighbours apart, and for that the order is enough. On space-only inputs without normalization and with orthonormal embeddings, $h(n,0)=(1-t_n)e_\#+t_ne_s$ with $t_n=n/(n+\gamma)$, so $e_s^\top h(n,0)=t_n$ reads the position along the segment directly. Choosing $w_c=2t_ce_s$ and $b_c=-t_c^2$ gives class $c$ the score $t_n^2-(t_n-t_c)^2$, which is largest at the closest position $c=n$, with $\|w_c\|<2$ and $|b_c|<1$ at every $L$. What grows with $L$ is what the head must resolve: the minimum winning margin over the counting range shrinks, and in finite precision a small enough margin is hard to distinguish. Bare pre-LN classification also escapes the cap of Theorem~\ref{thm:bare}, since the cap limits how evenly the counts are spaced, not whether they are ordered.

\section{Empirical results}
\label{sec:empirical}

\begin{figure}[H]
  \centering
  \includegraphics[width=\linewidth]{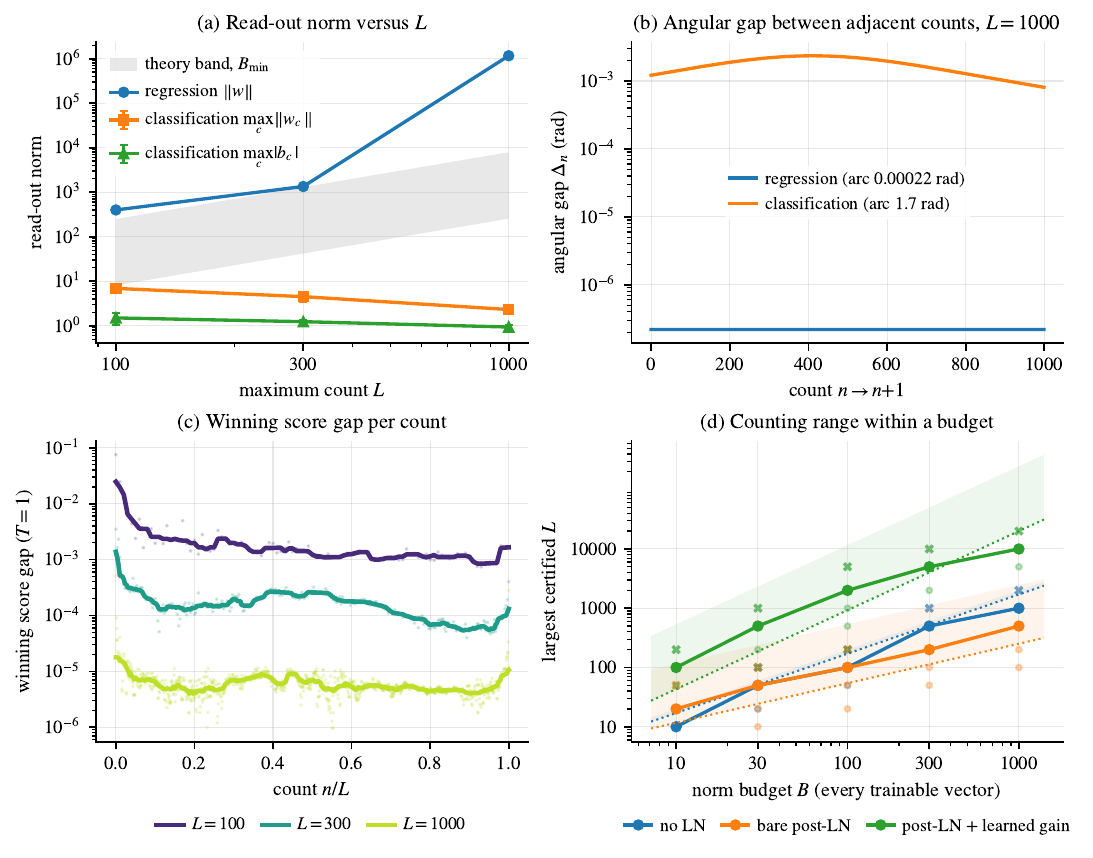}
  \caption{\textbf{What each read-out costs, and how far each variant counts within a budget.}
    Softmax attention, $d=16$, trained from random initialization by L-BFGS (Appendix~\ref{app:postln-protocol}). Panels (a--c) show runs that pass the full-domain check (all seeds pass at $L\le300$, one of three classifier seeds at $L=1000$).
    (a) Regression needs a read-out norm that grows with $L$ (gray: the range Theorem~\ref{thm:post}(a) allows for $B_{\min}$); the classifier's class vectors and biases stay $O(1)$ under our parameterization.
    (b) Where the two heads put the count representations before the read-out, on the letter-free slice: regression on a $2\times10^{-4}$ rad arc, classification over $1.7$ rad.
    (c) The classifier's cost is elsewhere: the smallest winning score gap over letter counts, per count, shrinks with $L$.
    (d) Largest $L$ that rounds correctly on the full domain when every trainable vector is capped at $B$. Dotted lines are the constructions; shaded regions the ranges the theorems allow.}
  \label{fig:post-ln-heads}
\end{figure}

\paragraph{Regression versus classification.}
Trained from the same random initialization on the same bare post-LN model, the two read-outs land on the solutions the theory describes (Figure~\ref{fig:post-ln-heads}a--c). Regression's read-out norm grows with $L$, on or above the band of Theorem~\ref{thm:post}(a), and it gets there by packing all $1001$ counts into $2\times10^{-4}$ rad of arc, where the prediction is nearly linear in the count. The classifier keeps every class vector below norm $7$ and every bias below $2$, and spreads the counts over $1.7$ rad: it only has to give the correct class the largest score, not reproduce the count through one linear read-out. What shrinks instead is the winning score gap, from about $10^{-3}$ at $L=100$ to $10^{-6}$ at $L=1000$. Gaps this small give cross-entropy almost no probability separation, and training did not succeed without a learnable positive logit multiplier. That scalar does not change the prediction, and every norm and gap we report has it divided out (Appendix~\ref{app:postln-protocol}).

\paragraph{Counting range under a budget.}
Two of the three exponents in Table~1 are recovered by training (panel~d). We cap every trainable vector at $B$ and find the largest $L$ that still rounds correctly. The fitted slopes are $1.00$ without normalization and $0.68$ for bare post-LN, against the predicted $1$ and $2/3$. The learned gain reaches at least ten times the length of either alternative at every budget from $100$ on, but its fitted slope is $1.00$ rather than the predicted $4/3$, so this exponent is not resolved. At large budgets, the measurement falls short of what the model can do. At $B = 1000$, the construction guarantees exactly $20{,}000$, but that grid cell failed. At $B = 3000$, the runs reach only $10{,}000$ against a guaranteed $86{,}535$ (Table~3). At these budgets, the optimizer, not the model, sets the limit.

\paragraph{Training the unnormalized model.}
We also train the unnormalized softmax model by gradient descent (Appendix~\ref{empirical-appendix}). The parameter norms reach the lower bound of Theorem~\ref{thm:nonorm} early in training, but the saturation scale $\gamma$ grows far more slowly, and most models come close to correct rounding without reaching it. This suggests that the score gap, rather than the norm, is the harder quantity for gradient descent to learn. Appendix~\ref{representation-crowding-LLM} reports a qualitative parallel in Qwen3-4B, where adjacent high-count representations are more crowded than low-count ones.

\section{Conclusion}
\label{sec:discussion}

We studied the cost of counting in a single attention layer, measured by the smallest norm budget $\Bmin$ at which rounding is correct up to length $L$. This cost depends on where layer normalization is placed. Softmax places the counts on a segment and post-normalization places them on an arc, and the arc needs less magnification because it stays nearly straight for longer. As a result, learned post-LN needs the smallest budget of all the softmax configurations, $\Theta(L^{3/4})$ at fixed dimension, even though normalization discards the magnitude that carries the count. Bare pre-LN is the one configuration that no budget can rescue, and a classification head keeps every norm bounded but stores one vector per count.

\paragraph{Implications for length generalization.}
Our results suggest one reason why a model trained on inputs up to length $L$ can fail on longer inputs. During training, the parameters can grow to the norm that counting to $L$ requires, and once training stops, the norms are fixed. In the configurations where larger norms help, $\Bmin$ grows polynomially with length, so a budget close to $\Bmin(L)$ is smaller than $\Bmin(L')$ once $L'$ is a large enough multiple of $L$, and then no setting of the parameters within that budget counts correctly to $L'$, whatever the training data. This differs from the usual explanation of out-of-distribution failure, in which a model learns a heuristic that only works on the lengths it was trained on. That explanation concerns what training finds, while ours concerns what exists within the budget. A model that counts correctly to $L'$ does exist, as our constructions show, but only at a larger norm. Since our analysis covers a single attention layer, this is a hypothesis about larger models rather than a result about them.

\paragraph{Limitations.}
Our results concern a single attention head with identity projections and layer normalization with zero stability constant. Our constructions show what this single-layer, single-head model can represent. How training finds these representations is a separate question.

\clearpage
\subsection*{AI use statement}
Generative AI was used to revise and expand on proofs, including working out details and edge cases. It was also used to review mathematical claims, to help write the code for the experiments, to propose and check revisions to the text, and to assist with LaTeX editing and compilation. All proofs were checked, and all code tested, by the authors. The problem formulation, the statements of Theorems 1--4, the choice of experiments, and the interpretation of their results are the authors' own. We take full responsibility for the content of this paper.

\subsection*{Reproducibility statement}
The theoretical model and budget conventions are specified in Section~\ref{sec:setup}. Proofs and constructions appear in Appendices~\ref{app:scope}--\ref{app:consequences}. Experimental details appear in Appendix~\ref{empirical-appendix}.

\bibliographystyle{iclr2027_conference}
\bibliography{iclr2026_conference}
\clearpage
\appendix

\section{Conventions and scope}
\label{app:scope}

All vector norms are Euclidean. Unless explicitly stated otherwise,
$W_Q=W_K=W_V=I$, there is no residual connection, and regression has no
separate read-out bias. The input domain is
$\mathcal T_L=\{(n,m)\in\mathbb N_0^2:n+m\le L\}$, including $(0,0)$.
The boundary token is additional to these $L$ content positions.
We use $d\ge3$ and, for the joint norm theorems, $L\ge4$.
Lemmas with weaker length assumptions state them explicitly.

Write $P=I-\mathbf1\mathbf1^\top/d$ and
$\operatorname{LN}_0(x)=\sqrt d\,Px/\|Px\|$.
The normalization uses $\varepsilon=0$ and is required to be defined on
every input where it is applied. Learned normalization is
$g\odot\operatorname{LN}_0(x)+\beta$.
Each trainable embedding, read-out, gain, and normalization bias is
individually bounded by $B$ whenever it is present.
Fixed identity matrices are not trainable resources.

To avoid rounding ties and unattained boundary infima, define
\[
 B_{\min}(L,d)=\inf\{B>0:\text{there exists an admissible }\theta
 \text{ of budget }B\text{ with }E(\theta;L)<1/2\}.
\]
The infimum of the empty set is $+\infty$.
This is equivalent to using $E^\star(L,d;B)<1/2$ inside the set.
All asymptotic constants below are absolute unless the statement
explicitly holds additional parameters fixed.

\section{A scalar softmax lemma}
\label{app:softmax-lemma}

\begin{lemma}[Three-count obstruction]
\label{app:scalar}
Let $L\ge2$, $\gamma>0$, and
$f(n)=(n\alpha+\gamma r)/(n+\gamma)$ for $0\le n\le L$.
If $|f(n)-n|\le\delta$ at these integer counts, then, for every
$1\le j<L$,
\begin{align}
 \delta&\ge\frac{j(L-j)}{2(j+\gamma)}
 \ge0,\label{app:three-count}\\
 |r|&\le\delta,\qquad
 \alpha\ge L-\delta+\gamma(1-2\delta/L).
 \label{app:alpha-bound}
\end{align}
In particular,
\[
 \delta\ge\frac{L^2-1}{4L+8\gamma},
 \qquad
 \delta<\frac12\ \Longrightarrow\quad
 \begin{cases}
 \gamma>\lfloor (L-1)^2/4\rfloor,\\
 \alpha>(L^2-L)/4.
 \end{cases}
\]
\end{lemma}

\begin{proof}
Put $t=j(L+\gamma)/(L(j+\gamma))\in(0,1)$.
Direct substitution gives $f(j)=(1-t)f(0)+tf(L)$.
The three error bounds imply $tL-\delta\le f(j)\le j+\delta$,
and therefore $2\delta\ge tL-j=j(L-j)/(j+\gamma)$.
Taking $j=\lfloor L/2\rfloor$ proves the displayed uniform lower bound.
The tests at $0,L$ give \eqref{app:alpha-bound}.
If $\delta<1/2$, \eqref{app:three-count} gives
$\gamma>j(L-j-1)$; maximize this over integers $j$.
Also the weaker bound
$\delta>(L^2-1)/(8(L+\gamma))$ gives
$\gamma>(L^2-1)/4-L$. Substituting into
\eqref{app:alpha-bound} yields
$\alpha>(L^2-L+1+1/L)/4>(L^2-L)/4$.
\end{proof}

This lemma holds for every positive $\gamma$, including $\gamma<1$.
It applies to both unnormalized and pre-normalized softmax regression by
restricting to $m=0$ and taking $\alpha=w^\top z_s$, $r=w^\top z_\#$,
where $z$ denotes the values entering attention.

\section{Proof of Theorem~\ref{thm:nonorm}}
\label{app:nonorm}

\subsection{Linear attention}
\label{app:linear}

\begin{proposition}
Unnormalized dot-product attention counts exactly at finite parameters
independent of $L$. For $L\ge2,d\ge3$,
$B_{\min}^{\rm lin}(L,d)=\Theta(d^{1/8})$.
\end{proposition}
\begin{proof}
The prediction is affine:
$\hat y(n,m)=c_s(w^\top e_s)n+c_\ell(w^\top e_\ell)m+
c_\#(w^\top e_\#)$.
Exact counting is equivalent to its three coefficients being $1,0,0$;
necessity follows by evaluating $(0,0),(1,0),(0,1)$.
Let $f_1,f_2,f_3$ be coordinate unit vectors and $t=d^{1/8}$.
Choose
\[
 e_\#=tf_1,\quad e_s=t(f_1+f_2),\quad
 e_\ell=t(f_1+f_3),\quad w=tf_2.
\]
These embeddings are linearly independent and
$c_s(w^\top e_s)=t^4/\sqrt d=1$; the other two coefficients vanish.
All vector norms are at most $\sqrt2\,d^{1/8}$.
Conversely, error below $1/2$ at $(0,0),(L,0)$ forces the slope to
exceed $1-1/L\ge1/2$. Its absolute value is at most
$B^4/\sqrt d$, giving $B\ge2^{-1/4}d^{1/8}$ after taking infima.
\end{proof}

\subsection{Softmax: approximation and joint norm cost}
\label{app:softmax}

\begin{proposition}
For $L\ge2,d\ge3,\gamma\ge1$,
\[
 \frac{3}{32}\frac{L^2}{L+\gamma}
 \le E^\star(\gamma,L,d)\le\frac{L^2}{L+\gamma}.
\]
No finite parameters give exact regression when $L\ge2$, but
$E^\star(L,d)=0$. For $L\ge4$,
\[
 B_{\min}^{\rm soft}(L,d)
 =\Theta\!\left(\max\{L,d^{1/4}\sqrt{\log L}\}\right).
\]
\end{proposition}

\begin{proof}
The lower error bound follows from Lemma~\ref{app:scalar}, since
$(L^2-1)/L^2\ge3/4$ and $4L+8\gamma\le8(L+\gamma)$.
For the upper bound take
\[
 e_s=w=\sqrt\gamma f_1,\qquad
 e_\#=d^{1/4}\sqrt{\log\gamma}\,f_2,\qquad e_\ell=f_3.
\]
At $\gamma=1$ this simply uses $e_\#=0$, which is allowed without LN.
The scores are $c_s=c_\ell=0$, $c_\#=\log\gamma$, and
\begin{equation}
 \hat y(n,m)=\frac{\gamma n}{n+m+\gamma},\qquad
 n-\hat y(n,m)=\frac{n(n+m)}{n+m+\gamma}.
 \label{app:basic-soft-family}
\end{equation}
Its worst error is $L^2/(L+\gamma)$, attained at $(L,0)$.
Letting $\gamma\to\infty$ proves the zero infimum.
The strictly positive lower bound in Lemma~\ref{app:scalar}, which
also covers $\gamma<1$, rules out exact finite-parameter regression.

For rounding, $|\alpha|\le B^2$ and
Lemma~\ref{app:scalar} give $B=\Omega(L)$.
The identity projections also give
$\log\gamma=(\|e_\#\|^2-e_\#^\top e_s)/\sqrt d\le2B^2/\sqrt d$.
For $L\ge4$, rounding requires
$\gamma>\lfloor(L-1)^2/4\rfloor\ge L^2/8\ge\sqrt L$.
Thus $B\ge\tfrac12d^{1/4}\sqrt{\log L}$.
Finally, the construction above with $\gamma=2L^2$ has error strictly
below $1/2$ and budget
$\max\{\sqrt2 L,d^{1/4}\sqrt{\log(2L^2)},1\}$.
\end{proof}

\subsection{The optional calibrated construction}
\label{app:one-knob}

\begin{lemma}
For $\gamma>1$ and $u\in[0,L]$, choose orthogonal embeddings with
$\|e_s\|^2=\gamma+u$, $\|e_\#\|^2=\sqrt d\log\gamma$,
$\|e_\ell\|=1$, and $w=e_s$. Then
\[
 E=\max\left\{
 \max_{n\in\{0,\ldots,\lfloor u\rfloor\}}
 \frac{n(u-n)}{n+\gamma},\quad
 \frac{L(L-u)}{L+\gamma}\right\}
 \le\max\left\{\frac{u^2}{4\gamma},
 \frac{L(L-u)}{L+\gamma}\right\}.
\]
\end{lemma}
\begin{proof}
Writing $S=n+m$, the signed error is $n(S-u)/(S+\gamma)$.
For fixed $S$ its absolute value is largest at $n=S$, so $m=0$ suffices.
For $S\ge u$ the maximum is at $S=L$; for $S\le u$ use
$S(u-S)\le u^2/4$.
\end{proof}

For example, $\gamma=(6-4\sqrt2)L^2$ and $u=(2\sqrt2-2)L$ give
error strictly below $1/2$: each positive term $n(u-n)/(n+\gamma)$ is strictly below $u^2/(4\gamma)=1/2$, and the endpoint error is also strictly below $1/2$. The actual construction budget is
\[
 \max\left\{\sqrt{\gamma+u},d^{1/4}\sqrt{\log\gamma},1\right\}.
\]
At fixed $d$ this is $(2-\sqrt2)L+1/\sqrt2+O(L^{-1})$ as
$L\to\infty$. This is a sufficient construction cost, not a proof of
the globally optimal leading constant or the finite-$L$ optimum.

If $B^2\ge2L$ and $B^2\ge\sqrt d\log(B^2)$, the same lemma with
$u=L$, $\gamma=B^2-L$ gives $E^\star\le L^2/(2B^2)$.
Conversely, when $\delta\le L/4$, \eqref{app:alpha-bound} gives
$\gamma\le2B^2$, and \eqref{app:three-count} gives
$\delta\ge L^2/(32B^2)$ under these side conditions.
When $\delta>L/4$, that lower bound is immediate.
Thus the conditional error rate is $\Theta(L^2/B^2)$; the score-budget
side condition must not be dropped in a joint $(L,d,B)$ claim.

\section{Proof of Theorem~\ref{thm:post}}
\label{app:post}

\subsection{Geometry and realization}
\label{app:post-geometry}

\begin{lemma}[Normalized affine line]
On the letter-free slice, either attention rule followed by bare LN has
representation
$\sqrt d(A+nC)/\|A+nC\|$ for centered vectors $A,C$.
Conversely, any centered $A,C$ with $A+nC\ne0$ at all admissible
integer counts can be realized on the entire domain, independently of
$m$, with arbitrarily small raw embedding norms.
\end{lemma}

\begin{proof}
For softmax take $A=Pe_\#$ and $C=e^{c_s-c_\#}Pe_s$; the positive
denominator cancels in LN. For linear attention take
$C=(c_s/c_\#)Pe_s$; here $c_\#=\|e_\#\|^2/\sqrt d>0$ because
normalization at $(0,0)$ must be defined.

For the converse put $r=\mathbf1/\sqrt d$ and
$\kappa=\|A\|^2+1-A^\top C$. For $t>0$, choose
\[
 e_\#=t(A+r),\qquad e_s=t(C+\kappa r),\qquad e_\ell=0.
\]
Then $c_s=c_\#=t^2(\|A\|^2+1)/\sqrt d>0$.
For either rule, $Ph(n,m)$ is a positive multiple of $A+nC$.
Taking $t$ sufficiently small enforces any positive embedding budget.
\end{proof}

In particular, $A=u,C=\omega v$ with centered orthonormal $u,v$ gives
\begin{equation}
 \operatorname{LN}_0(h(n,m))=
 \sqrt d(\cos\theta_n\,u+\sin\theta_n\,v),\qquad
 \theta_n=\arctan(\omega n),\quad \omega>0.
 \label{app:arc}
\end{equation}
The adjacent angle is positive and decreases with $n$, because
$\theta'(x)=\omega/(1+\omega^2x^2)$ decreases for $x\ge0$.
The chord length is $2\sqrt d\sin((\theta_{n+1}-\theta_n)/2)$.
This particular arc formula requires $A\perp C$; arbitrary directions
are covered by the normalized-line formula instead.

\subsection{The four-count lower bound}
\label{app:four-count}

\begin{lemma}[Four-count bound]
\label{lem:four-count}
Let
\[
 f(n)=v^\top\frac{A+nC}{\|A+nC\|_2},\qquad
 M=\|v\|_2>0,
\]
where the denominator is nonzero for $n=0,\ldots,L$, and $L\ge3$.
If $\delta=\max_{0\le n\le L}|f(n)-n|$ and
$k=\lfloor L/3\rfloor$, then
\begin{equation}
 \delta\ge
 \min\left\{\frac{k}{16},\frac{k^3}{30M^2}\right\}.
 \label{eq:four-count}
\end{equation}
\end{lemma}

\begin{proof}
Rescale $A,C$ by $\|A\|_2$ so that $\|A\|_2=1$, without changing $f$.
Write
\[
 a=v^\top A,\quad b=v^\top C,\quad
 r_i=\|A+ikC\|_2,\quad e_i=f(ik)-ik
 \qquad(i=1,2,3).
\]
Then $|a|=|f(0)|\le\delta$, $|e_i|\le\delta$, and
\begin{equation}
 r_i=\frac{a+bik}{ik+e_i}.
 \label{eq:four-ri}
\end{equation}
If $\delta\ge k/16$ there is nothing to prove. Otherwise put
$\eta=\delta/k<1/16$. The denominators in \eqref{eq:four-ri}
are positive, and
\[
 r_i\le\frac{|b|+\eta}{1-\eta}.
\]
The affine identity $A=2(A+kC)-(A+2kC)$ gives
\[
 1\le2r_1+r_2\le\frac{3(|b|+\eta)}{1-\eta},
 \qquad |b|\ge\frac{1-4\eta}{3}\ge\frac14.
\]
Since $a+bk>0$, we also have $b>-\eta$; hence $b>0$.
Using $b\ge1/4$ in \eqref{eq:four-ri},
\[
 |r_i-b|
 =\frac{|a-be_i|}{ik+e_i}
 \le\frac{\eta(1+b)}{1-\eta}
 \le6b\eta.
\]
Therefore
\[
 |r_i^2-b^2|
 \le b^2(12\eta+36\eta^2)
 \le15b^2\eta.
\]
The squared norm $\|A+nC\|_2^2$ is quadratic with coefficient
$\|C\|_2^2$. Its second difference at $k,2k,3k$ is consequently
\[
 2k^2\|C\|_2^2
 =r_3^2-2r_2^2+r_1^2
 \le60b^2\eta.
\]
It follows that $\|C\|_2^2\le30b^2\delta/k^3$.
Finally, Cauchy--Schwarz gives $b^2\le M^2\|C\|_2^2$.
Cancelling $b^2>0$ yields $\delta\ge k^3/(30M^2)$, as required.
\end{proof}

This argument permits arbitrary $A^\top C$.
Appendix~\ref{app:affine-post} gives a bound that also permits
a scalar offset.

\subsection{Bare LN and learned gain with zero bias}
\label{app:post-rates}

\begin{proposition}
For either attention rule and $L\ge4,d\ge3,B>0$,
\begin{align*}
 E^\star_{\rm bare\ post}(L,d;B)
 &=\Theta\!\left(\min\{L,L^3/(dB^2)\}\right),&
 B_{\min}^{\rm bare\ post}
 &=\Theta(L^{3/2}/\sqrt d),\\
 E^\star_{\rm gain\ post}(L,d;B)
 &=\Theta\!\left(\min\{L,L^3/(dB^4)\}\right),&
 B_{\min}^{\rm gain\ post}
 &=\Theta(L^{3/4}/d^{1/4}).
\end{align*}
The second line assumes $\beta=0$ and separately bounds $g,w$ by $B$.
Both unrestricted error infima are zero. Exact regression is impossible
at finite parameters for $L\ge3$ in these two bias-free cases.
\end{proposition}

\begin{proof}
For bare LN, use \eqref{app:arc} with
$w=v/(\omega\sqrt d)$. Then
\[
 \hat y(n,m)=\frac{n}{\sqrt{1+\omega^2n^2}},\qquad
 0\le n-\hat y(n,m)\le\frac{\omega^2n^3}{2}.
\]
The upper inequality is strict for $n>0$, since
$1-(1+x)^{-1/2}<x/2$ for $x>0$.
Set $\omega=1/(\sqrt d B)$; together with the zero read-out this
proves the bare upper error bound. The realization lemma makes the
embedding budget arbitrarily small.

For learned gain, choose
$v=(f_1-f_2)/\sqrt2$, $u=(f_1+f_2-2f_3)/\sqrt6$,
$g=(B/\sqrt2,\allowbreak B/\sqrt2,0,\ldots,0)$, and $w=Bv$.
Then $g\odot w=(B^2/\sqrt2)v$.
Taking $\omega=\sqrt2/(\sqrt d B^2)$ gives the same prediction
and error at most $L^3/(dB^4)$.

For the lower bounds apply the four-count lemma with
$M\le\sqrt d B$ in the bare case, and
$M\le\sqrt d\|g\odot w\|\le\sqrt d B^2$ in the gain case.
Since $k\ge L/5$, it gives at least
$3750^{-1}\min\{L,L^3/M_{\max}^2\}$, where $M_{\max}=\sqrt d B$
or $\sqrt d B^2$, respectively.
The zero effective read-out has error $L$ and causes no exception.

The upper rounding bounds follow by setting $\omega=L^{-3/2}$ and
choosing the corresponding read-out or gain/read-out factors.
For the lower bounds, when $L\ge24$, $k\ge8$ and $\delta<1/2$
forces the small-error branch of the lemma; thus $M^2>k^3/15$
and $M=\Omega(L^{3/2})$.
For $4\le L<24$, $M\ge f(L)>L-1/2$ gives the same order with an
absolute smaller constant. Insert the appropriate bound on $M$.
In particular,
\[
 \frac{L^{3/2}}{\sqrt{960d}}
 \le B_{\min}^{\rm bare\ post}(L,d)
 \le\frac{L^{3/2}}{\sqrt d}
 \qquad(L\ge4).
\]
For $L\ge24$, $\lfloor L/3\rfloor\ge L/4$ gives the lower constant;
for $4\le L<24$, $M>L-1/2$ is stronger than
$M\ge L^{3/2}/\sqrt{960}$.

Finally, exact counting with no scalar offset has the form
$(a+bn)/r_n=n$, where $r_n=\|A+nC\|$.
The count $0$ forces $a=0$; counts $1,2,3$ force
$r_1=r_2=r_3=b>0$.
The quadratic $r_n^2-b^2$ has three roots and is identically zero.
Its leading coefficient is $\|C\|^2$, forcing $C=0$ and then $b=0$,
a contradiction. Approximation follows by letting $\omega\downarrow0$.
\end{proof}

\subsection{Learned gain and bias}
\label{app:affine-post}

We first extend the normalized-line lower bound to allow an
arbitrary scalar shift.

\begin{lemma}[Four-count bound with a scalar shift]
\label{lem:affine-four-count}
Let $L\ge4$, $M=\|v\|>0$, and
\[
 f(x)=t+v^\top\frac{A+xC}{\|A+xC\|},
\]
where the denominator is nonzero at the integer counts
$0,\ldots,L$. Put
$\delta=\max_{0\le n\le L}|f(n)-n|$ and
$k=\lfloor L/4\rfloor$. Then
\begin{equation}
 \delta\ge
 \min\left\{\frac{k}{4},\frac{k^3}{24M^2}\right\}.
 \label{eq:affine-four-count}
\end{equation}
The bound is independent of $t$.
\end{lemma}

\begin{proof}
If $\delta\ge k/4$, there is nothing to prove. Otherwise
the values at $0,k,3k,4k$ are strictly increasing.
Thus $A,C$ are linearly independent: a normalized collinear affine
line has at most two values wherever it is defined.
Consequently $f$ is smooth on the whole real line.

Choose orthonormal coordinates in $\operatorname{span}\{A,C\}$.
With an affine coordinate $s=\kappa x+\lambda$, we can write
\[
 f(x)=\Phi(s),\qquad
 \Phi(s)=t+\frac{a+bs}{\sqrt{1+s^2}},\qquad
 \kappa>0,\quad a^2+b^2\le M^2.
\]
Indeed, take one coordinate along $C$ and the other along the
nonzero component of $A$ orthogonal to $C$.
The numerator of $f'$ is therefore affine, and its
denominator is positive.

We claim that, wherever $f'>0$, the function
$Z=(f')^{-1/2}$ satisfies
\begin{equation}
 Z''\ge\frac{3}{2M^2}Z^{-3}.
 \label{eq:affine-convexity}
\end{equation}
To verify this, put $Q=1+s^2$, $p=b-as$, and
$R^2=a^2+b^2$. Where $\Phi'>0$, we have
$p>0$, $\Phi'=p/Q^{3/2}$, and
$\zeta=(\Phi')^{-1/2}=Q^{3/4}p^{-1/2}$.
Direct differentiation gives
\[
 \zeta''\zeta^3
 =\frac{3Q(R^2Q+p^2)}{4p^4}
 =\frac{3}{4R^2}\frac{1+r}{r^2},
 \qquad r=\frac{p^2}{R^2Q}\in(0,1].
\]
Here $r\le1$ follows from Cauchy--Schwarz.
Since $(1+r)/r^2\ge2$ and $R\le M$, the last expression
is at least $3/(2M^2)$.
Moreover, $Z(x)=\kappa^{-1/2}\zeta(\kappa x+\lambda)$,
so $Z''Z^3=\zeta''\zeta^3$. This proves
\eqref{eq:affine-convexity}.

Set $\epsilon=2\delta/k<1/2$.
The mean value theorem on $[0,k]$ and $[3k,4k]$
gives points $\xi\in(0,k)$ and $\eta\in(3k,4k)$ with
\[
 1-\epsilon\le f'(\xi),f'(\eta)\le1+\epsilon.
\]
Because the numerator of $f'$ is affine, $f'>0$ throughout
$[\xi,\eta]$. By \eqref{eq:affine-convexity}, $Z$ is convex
there. Its endpoint values are at most
$U=(1-\epsilon)^{-1/2}$, so $Z\le U$ throughout and
\[
 Z''\ge K,\qquad K=\frac{3}{2M^2}U^{-3}.
\]
Applying convexity to $Z(x)-Kx^2/2$ gives, for
$x\in[k,3k]$,
\[
 Z(x)\le U-\frac K2(x-\xi)(\eta-x)
       \le U-\frac K2(x-k)(3k-x).
\]
The final upper bound is positive because $Z(x)>0$.
Using $(U-q)^{-2}\ge U^{-2}+2qU^{-3}$ for
$0\le q<U$, we obtain
\[
 f'(x)=Z(x)^{-2}
 \ge U^{-2}+KU^{-3}(x-k)(3k-x).
\]
Integrating over $[k,3k]$ yields
\[
 \begin{split}
 f(3k)-f(k)
 &\ge 2kU^{-2}+\frac43 k^3KU^{-3}\\
 &=2k-4\delta+
   \frac{2k^3}{M^2}(1-\epsilon)^3.
 \end{split}
\]
The error bounds also give $f(3k)-f(k)\le2k+2\delta$.
Therefore
\[
 \delta\ge\frac{k^3}{3M^2}(1-\epsilon)^3
 \ge\frac{k^3}{24M^2},
\]
which proves the lemma.
\end{proof}

\begin{proposition}[Learned affine post-normalization]
\label{prop:affine-post}
For either attention rule, $L\ge4$, $d\ge3$, and $B>0$,
post-normalization with learned gain and bias, with each of
$e_s,e_\ell,e_\#,g,\beta,w$ bounded by $B$, satisfies
\[
 E^\star_{\rm affine\ post}(L,d;B)
 =\Theta\!\left(\min\left\{L,\frac{L^3}{dB^4}\right\}\right),
 \qquad
 B_{\min}^{\rm affine\ post}(L,d)
 =\Theta\!\left(\frac{L^{3/4}}{d^{1/4}}\right).
\]
The constants are absolute. The unrestricted error infimum is
zero, but no finite parameters achieve exact regression.
\end{proposition}

\begin{proof}
The gain-only constructions in Appendix~\ref{app:post-rates}
remain admissible with $\beta=0$, proving the upper error and
rounding bounds.

For the lower bounds, restrict to $m=0$ and write
\[
 \hat y(n,0)=t+v^\top\frac{A+nC}{\|A+nC\|},
 \qquad
 t=w^\top\beta,\quad v=\sqrt d\,(g\odot w).
\]
Then $M=\|v\|\le\sqrt d\,B^2$.
If $M=0$, the prediction is constant on this slice and its
worst error is at least $L/2$.
Otherwise Lemma~\ref{lem:affine-four-count} applies.
Since $\lfloor L/4\rfloor\ge L/7$ for $L\ge4$, it gives
\[
 E(\theta;L)
 \ge \min\left\{\frac{L}{28},\frac{L^3}{8232M^2}\right\}
 \ge\frac1{8232}
 \min\left\{L,\frac{L^3}{dB^4}\right\}.
\]

For the rounding lower bound, suppose $E(\theta;L)<1/2$.
If $L\ge8$, then $k=\lfloor L/4\rfloor\ge2$, so the
first branch of \eqref{eq:affine-four-count} is at least
$1/2$. Hence $M^2>k^3/12\ge L^3/4116$.
For $4\le L\le7$, the endpoint inequality
\[
 L-1<\hat y(L,0)-\hat y(0,0)\le2M
\]
gives the same bound with the stated constant.
Thus, for every $L\ge4$,
\[
 B_{\min}^{\rm affine\ post}(L,d)
 \ge\left(\frac1{4116}\right)^{1/4}
       \frac{L^{3/4}}{d^{1/4}}.
\]

The positive error lower bound excludes exact regression at
finite parameters. Letting the budget grow in the upper
construction makes the error tend to zero.
No part of the lower-bound argument uses a bound on the
scalar shift $t$.
\end{proof}

\subsection{Why a small-length qualification is necessary}
\label{app:counterexamples}

For $L=2$, even bare post-LN can regress exactly. Take centered
orthonormal $u,v$ and set
\[
 A=u,\qquad C=-\frac3{10}u+\frac{\sqrt{11}}{10}v,
 \qquad w=\sqrt{\frac{60}{11d}}\,v.
\]
Then $\|A+C\|^2=\|A+2C\|^2=3/5$, and
$\sqrt d\,w^\top(A+nC)/\|A+nC\|=n$ for $n=0,1,2$.
The realization lemma supplies finite embeddings and removes all letter
dependence. Thus both the impossibility statement and positive post-LN
error bound cannot be extended to $L=2$.

The dimension-uniform unnormalized softmax norm formula also needs a
length qualification. At $L=2$, let the three embeddings be coordinate
unit vectors and take
$w=(16e_s-2e_\ell-2e_\#)/5$.
Its norm is $\sqrt{264}/5$, independent of $d$.
Here $\gamma=e^{1/\sqrt d}\in(1,e^{1/\sqrt3}]$ and
\[
 \hat y(n,m)=
 \frac{(16/5)n-(2/5)m-(2/5)\gamma}{n+m+\gamma}.
\]
At $n=0$ the error is $2/5$.
Direct substitution at $(1,0),(1,1),(2,0)$ gives errors below $1/2$;
the strongest upper restriction is $\gamma<34/19$, which holds for
$d\ge3$. A necessary cost $\Omega(d^{1/4}\sqrt{\log2})$ is therefore
false. The joint theorem above uses $L\ge4$.

\section{Proofs for pre-normalization}
\label{app:pre}

\subsection{Bare pre-normalization: both sides of the bound}
\label{app:bare-pre}

\begin{proposition}
For $L\ge2,d\ge3$ and $\Gamma_d=e^{2\sqrt d}$,
\[
 E^\star_{\rm bare\ pre}(L,d)
 =\Theta\!\left(\frac{L^2}{L+\Gamma_d}\right).
\]
The largest length attainable with strict rounding at any finite norm,
$L_{\rm cap}(d)$, satisfies $L_{\rm cap}(d)=\Theta(e^{\sqrt d})$.
\end{proposition}
\begin{proof}
Bare normalized values $z_v$ have norm $\sqrt d$, so
$c_\#=\sqrt d$, $|c_s|\le\sqrt d$, and
$1\le\gamma\le\Gamma_d$. Lemma~\ref{app:scalar} gives
$E^\star\ge(L^2-1)/(4L+8\Gamma_d)$.

For the upper bound, take orthonormal $u,v\in\mathbf1^\perp$
and, for $0<\phi<\pi$, set
\[
 z_\#=\sqrt d\,u,\quad z_\ell=-\sqrt d\,u,\quad
 z_s=\sqrt d(-\cos\phi\,u+\sin\phi\,v).
\]
All are realized by nonzero raw embeddings in those directions.
Put $\gamma=e^{\sqrt d(1+\cos\phi)}$,
$q=e^{\sqrt d(\cos\phi-1)}\le1$, and
$w=\gamma v/(\sqrt d\sin\phi)$. Then
\[
 \hat y(n,m)=\frac{\gamma n}{n+qm+\gamma},\qquad
 0\le n-\hat y(n,m)\le\frac{L^2}{L+\gamma}.
\]
Letting $\phi\downarrow0$ gives
$E^\star\le L^2/(L+\Gamma_d)$.
Only an infimum is claimed at the limiting antipodal configuration.

The lower bound rules out strict rounding once
$L\ge1+\sqrt{2+4\Gamma_d}$.
For the converse, choose $\phi>0$ so that $\gamma>\Gamma_d/2$.
Then every integer $L\le\sqrt{\Gamma_d}/2$ has error below $1/2$.
Integer rounding of this length bound changes only absolute constants
for $d\ge3$, proving the cap's order.
\end{proof}

\subsection{Learned gain and bias, including dimensions 3--5}
\label{app:learned-pre}

The main text states the learned pre-LN rate in Section~\ref{sec:ln:pre} and Table~\ref{tab:regression}. The full statement is the following.

\begin{proposition}
For $L\ge2,d\ge3,B>0$, learned pre-normalization with gain and bias satisfies $E^\star(L,d)=0$, although zero error is impossible at finite parameters, and under a norm budget $B$ on each of $e_s,e_\ell,e_\#,g,\beta,w$,
\[
 E^\star(L,d;B)=\Theta\!\left(
 \min\left\{L,\frac{L^2}{\sqrt d B^2}\right\}\right).
\]
Consequently $B_{\min}^{\rm pre}(L,d)=\Theta(L/d^{1/4})$ for $L\ge4$. The constants are independent of $L$, $d$, and $B$.
\end{proposition}

\begin{proof}
Put $K=\sqrt d B^2$. Since
$\|z_s\|\le(\sqrt d+1)B$, we have $|\alpha|\le2K$.
If $\delta\le L/4$, \eqref{app:alpha-bound} gives
$\gamma\le2\alpha\le4K$; hence
$\delta\ge(L^2-1)/(4L+32K)$.
This is an absolute constant times $\min\{L,L^2/K\}$.
If $\delta>L/4$ the same conclusion holds directly.

For an upper bound it suffices to realize
\eqref{app:basic-soft-family} with budget squared at most
$4\gamma/\sqrt d$ for every $\gamma>1$.
The raw embeddings below can always be scaled to any prescribed
positive norm without changing their bare normalized directions.

For $d\ge6$, put
$p=(f_1-f_2)/\sqrt2$, $q=(f_3-f_4)/\sqrt2$,
$r=(f_5-f_6)/\sqrt2$. Use raw directions $p,q,r$ for
spaces, boundary, and letters, respectively, and choose
\[
 W^2=\frac{\sqrt2\gamma}{\sqrt d},\quad G=W/\sqrt2,
 \quad H^2=\frac{\log\gamma}{\sqrt d},\quad
 g=(G,G,H,H,0,\ldots,0),\quad \beta=0,\quad w=Wp.
\]
Then $z_s=G\sqrt d\,p$, $z_\#=H\sqrt d\,q$, $z_\ell=0$,
and the desired prediction follows. Moreover
$\|g\|^2=(\sqrt2\gamma+2\log\gamma)/\sqrt d
\le4\gamma/\sqrt d$, and $\|w\|\le\|g\|$.
The zero letter value occurs \emph{after} a well-defined normalization.

For $3\le d\le5$, let $a=\sqrt{d/6}$ and choose the bare normalized
vectors, padded by zeros, as
\[
 x_\ell=a(1,1,-2),\quad x_s=a(-2,1,1),\quad
 x_\#=a(1,-2,1).
\]
They are centered and have norm $\sqrt d$.
Put $D=3a=\sqrt{3d/2}$ and
\begin{gather*}
 G=W=\sqrt{\gamma/D},\qquad
 H^2=\frac{2\log\gamma}{3\sqrt d},\\
 g=(G,H,0,\ldots,0),\quad
 \beta=(-aG,-aH,0,\ldots,0),\quad w=-Wf_1.
\end{gather*}
The resulting values are
$z_s=-DGf_1$, $z_\#=-DHf_2$, $z_\ell=0$.
Thus $c_s=c_\ell=0$, $c_\#=\log\gamma$, and
$w^\top z_s=\gamma$.
Also $\|\beta\|=a\|g\|\le\|g\|$ and
$\|g\|^2=(\sqrt{2/3}\gamma+(2/3)\log\gamma)/\sqrt d
\le4\gamma/\sqrt d$.

If $K>4$, set $\gamma=K/4$ in the applicable construction.
Its error is at most $L^2/(L+K/4)\le4\min\{L,L^2/K\}$.
If $K\le4$, the zero read-out gives $E=L$, which is at most twice
the claimed rate for $L\ge2$.
For rounding, Lemma~\ref{app:scalar} and $|\alpha|\le2K$
give the lower bound; setting $\gamma=2L^2$ above gives the upper.
Finally let $\gamma\to\infty$ for approximation, and use the positive
three-count obstruction to exclude exact regression.
\end{proof}

\subsection{Two useful extensions}
\label{app:pre-extensions}

Bare pre-normalization with \emph{linear} attention counts exactly at
every length. For orthonormal $p,q\in\mathbf1^\perp$, choose
$z_\#=z_\ell=\sqrt d\,q$, $z_s=\sqrt{d/2}(p+q)$, and $w=2p/d$.
Then $c_s=\sqrt{d/2}$ and the three prediction coefficients are $1,0,0$.
Raw embeddings can be scaled freely, so $B_{\min}\le 2/d$; conversely, error below $1/2$ at $(0,0),(L,0)$ forces $c_sw^\top z_s>1/2$, and $|c_sw^\top z_s|\le d\|w\|$, so $B_{\min}=\Theta(1/d)$.

Bare pre-normalization with \emph{softmax} has no finite counting cap when query/key
projections are unrestricted. Choose
$z_s=\sqrt d\,p$, $z_\#=\sqrt d\,q$, $z_\ell=-\sqrt d\,q$,
$W_Q=W_K=tI$, and $w=\gamma p/\sqrt d$, where
$\gamma=e^{t^2\sqrt d}$.
The prediction is $\gamma n/(n+m/\gamma+\gamma)$, whose error is at most
$L^2/(L+\gamma)\to0$ as $t\to\infty$.
This is an expressivity result and does not assert a matrix-norm cost.

\section{Proof of Theorem~\ref{thm:classification}}
\label{app:classification}

\begin{proposition}
For every finite $L$ and fixed $d\ge3$, softmax attention with no LN,
bare or learned pre-LN, and bare or learned post-LN admits an affine
classification head that predicts $n$ uniquely on all of $\mathcal T_L$.
Each parameter vector and each class bias is bounded independently of
$L$. The constants may depend on $d$.
\end{proposition}

\begin{proof}
Without normalization take orthonormal embeddings and
$\gamma=e^{1/\sqrt d}$. Then
\[
 h(n,m)=\frac{ne_s+me_\ell+\gamma e_\#}{n+m+\gamma}.
\]
Put $N=\max\{L,2\}$. For $0\le c\le L$, let
\[
 w_c=N^{-2}(\gamma c e_s-c^2e_\#/2),\qquad b_c=0.
\]
With $z=\gamma/(n+m+\gamma)>0$, the logit is
$z(cn-c^2/2)/N^2$ and the correct-class gap is
\begin{equation}
 s_n-s_c=\frac{z(n-c)^2}{2N^2}>0\quad(c\ne n).
 \label{app:quadratic-margin}
\end{equation}
The class-vector norms are at most $\gamma/N+1/2<2$.
This also covers $L=0,1$; the single class at $L=0$ is trivial.

For bare pre-LN, choose any three distinct centered points
$z_s,z_\ell,z_\#$ on the sphere of radius $\sqrt d$.
These normalized values are realized by taking the raw embeddings
equal to them.
They are not collinear: a line meets that sphere in at most two points.
Let $\gamma=e^{c_\#-c_s}$ and let
\[
 \lambda_\#=
 \frac{e^{c_\#}}
 {n e^{c_s}+m e^{c_\ell}+e^{c_\#}}>0
\]
be the boundary token's attention weight.
Writing $a=z_s-z_\ell$ and $b=z_\#-z_\ell$, we have
\[
 h(n,m)-z_\ell=\lambda_\#((n/\gamma)a+b).
\]
Since $a,b$ are linearly independent, the unit directions
\[
 u_n=\frac{(n/\gamma)a+b}{\|(n/\gamma)a+b\|}
\]
are distinct. Set $w_c=u_c$ and $b_c=-u_c^\top z_\ell$.
At an input of count $n$, the logit is $r u_c^\top u_n$, where
\[
 r=\lambda_\#\|(n/\gamma)a+b\|>0.
\]
It equals $r$ only at $c=n$ and is smaller at every other class by
strict Cauchy--Schwarz. Here $\|w_c\|=1$ and $|b_c|\le\sqrt d$.
Learned pre-LN includes this construction with $g=\mathbf1$ and
$\beta=0$.

For post-LN, fix any $\omega>0$ in \eqref{app:arc} and use
\[
 w_c=\cos\theta_c\,u+\sin\theta_c\,v,\qquad b_c=0.
\]
The logit is $\sqrt d\cos(\theta_n-\theta_c)$, uniquely maximized
at $c=n$. The embedding realization is independent of $L$ for fixed
$\omega$. Learned post-LN again includes the bare construction.
The post-LN argument also applies to linear attention.
\end{proof}

These constructions require suitable embeddings and gains.
For example, bare pre-LN with identity projections and $\gamma=1$
forces $z_s=z_\#$ and cannot distinguish $(0,0)$ from $(1,0)$.
Classification is not guaranteed at that saturation scale.

\section{Spacing, margins, and precision}
\label{app:precision}

On the letter-free softmax slice,
\[
 h(n,0)=e_\#+\lambda_n(e_s-e_\#),\qquad
 \lambda_n=\frac{n}{n+\gamma},
\]
and
\[
 \lambda_{n+1}-\lambda_n
 =\frac{\gamma}{(n+\gamma)(n+1+\gamma)}.
\]
These are exact identities. Thus the coordinate gap is of order
$1/\gamma$ when $n\ll\gamma$ and of order $\gamma/n^2$ when
$n\gg\gamma$, with the appropriate limiting regimes understood.

For the following margin comparisons, take $L\ge2$.
Spacing is not itself a logit margin.
For the nearest-center classifier on $t_n=n/(n+\gamma)$,
the correct-class score gap against class $c$ is $(t_n-t_c)^2$.
At fixed $\gamma$, its smallest adjacent gap is $\Theta(L^{-4})$.
For \eqref{app:quadratic-margin}, the minimum absolute margin is
\[
 \frac{\gamma}{2L^2(L+\gamma)}=\Theta(L^{-3})
\]
at fixed $d$. Its margin relative to the common logit scale $z$ is
$1/(2L^2)$.
For the post-LN sector classifier at fixed $\omega$, the minimum
margin is
\[
 \sqrt d\bigl[1-\cos(\theta_L-\theta_{L-1})\bigr]
 =\Theta(\sqrt d\,L^{-4}).
\]
The last statements follow from the exact spacing identities,
$\arctan'(x)=1/(1+x^2)$, and $1-\cos x\sim x^2/2$.

\begin{proposition}[A bound from adjacent representations]
Suppose adjacent-count representations $h_n,h_{n+1}$ are correctly
classified, $\|w_c\|\le B$, and their respective margins against each
other's class are $\mu_n,\mu_{n+1}>0$. Then
\[
 \mu_n+\mu_{n+1}\le2B\|h_n-h_{n+1}\|.
\]
In particular, a common guaranteed margin $\mu$ is at most
$B\|h_n-h_{n+1}\|$.
\end{proposition}

\begin{proof}
Add the two margin inequalities. The class biases cancel and the sum is
\[
 (w_n-w_{n+1})^\top(h_n-h_{n+1}).
\]
Apply Cauchy--Schwarz and
$\|w_n-w_{n+1}\|\le\|w_n\|+\|w_{n+1}\|\le2B$.
\end{proof}

\begin{proposition}[A sufficient perturbation condition]
Suppose $\|h\|\le H$, $\|w_c\|\le B$, and the exact winning margin
is at least $\mu>0$.
If representation, weight, and bias perturbations are bounded by
$\eta,\xi,\zeta$, respectively, and each computed logit has a further
arithmetic error at most $\tau$, classification is unchanged whenever
\[
 2(B\eta+H\xi+\eta\xi+\zeta+\tau)<\mu.
\]
\end{proposition}

\begin{proof}
Expand
\[
 (w_c+\Delta w_c)^\top(h+\Delta h)+b_c+\Delta b_c.
\]
Including the arithmetic error, each logit changes by at most
$B\eta+H\xi+\eta\xi+\zeta+\tau$.
The difference of two logits changes by at most twice that amount.
\end{proof}

\subsection{Storage of numerical parameters}
\label{app:storage}

We count bits used to store numerical parameters, interpreting
the model with those stored values in exact arithmetic.
A dense affine classifier explicitly stores its $L+1$ class
parameter tuples in a common $p$-bit scalar format.
The precision $p$ may differ between the classification and
regression models. Formula-generated rows are not stored
dense heads.

At a common scalar precision, the head storage is exactly
$dp$ bits for regression and $(L+1)(d+1)p$ bits for affine
classification. Three embeddings add $3dp$ bits to either
model. Learned gain and bias add $2dp$ bits, or $dp$ for
gain only. Thus the head-storage ratio at a common precision
is $(L+1)(1+1/d)=\Theta(L)$.

\begin{proposition}[Dense parameter storage]
\label{prop:dense-storage}
Suppose an affine classifier predicts every count
$0,\ldots,L$ uniquely somewhere on $\Tset$, and its dense
head stores each scalar in a common $p$-bit numerical format.
Then its head storage satisfies
\[
 S_{\rm cls}=(L+1)(d+1)p
 \ge (L+1)\log_2(L+1).
\]
For unnormalized softmax, at every fixed $d\ge3$ and
$L\ge4$, a bounded-norm classifier can achieve correct
classification on $\Tset$ with $O_d(L\log L)$ bits for
all parameters. A regressor can achieve strict rounding
on the same domain with $O_d(\log L)$ bits for all
parameters. Consequently the dense-classifier storage cost
is $\Theta_d(L\log L)$, while regression admits
$O_d(\log L)$ stored-parameter cost.
\end{proposition}

\begin{proof}
\textbf{Necessary classifier storage.}
The $L+1$ tuples $(w_c,b_c)$ must be pairwise distinct:
two identical tuples give identical logits on every input,
so neither class could win uniquely.
A tuple of $d+1$ scalars in a $p$-bit format has at most
$2^{p(d+1)}$ possible numerical values.
Thus
\[
 L+1\le2^{p(d+1)},\qquad
 p(d+1)\ge\log_2(L+1),
\]
which proves the lower bound.

\textbf{Sufficient classifier storage.}
Use coordinate unit embeddings in the unnormalized
construction of Appendix~\ref{app:classification}.
Then $\gamma=e^{1/\sqrt d}$ is fixed, $\|h\|\le1$,
and the class-vector norms are below $2$.
For $L\ge4$, \eqref{app:quadratic-margin} gives the
minimum margin
\[
 \mu=\frac{\gamma}{2L^2(L+\gamma)}
 =\Theta_d(L^{-3}).
\]
Round every head coordinate to a binary fixed-point grid
with coordinate error below $\mu/(4\sqrt d)$, and leave
the zero biases and unit embeddings exact.
Each class-vector perturbation has norm below $\mu/4$,
so each logit difference changes by less than $\mu/2$.
The maximizing class therefore remains correct and unique.
Since the coefficients are bounded, $O_d(\log L)$ bits
per scalar suffice, giving $O_d(L\log L)$ bits in total.
The quantized class vectors still have bounded norms.

\textbf{Sufficient regression storage.}
Let $f_1,f_2,f_3$ be coordinate unit vectors and put
\[
 b_0=\sqrt{\sqrt d\,\log(64L^2)}.
\]
Choose the smallest nonnegative integer $k$ such that
\[
 2^{-k}\le
 \min\left\{1,\frac{\sqrt d}{2b_0+1}\right\},
\]
and define
\[
 b=2^{-k}\lceil2^k b_0\rceil,\qquad
 \gamma=\exp(b^2/\sqrt d),\qquad
 a=\lceil\sqrt\gamma\rceil.
\]
The stored values $a,b$ are respectively an integer and
a binary rational; $\gamma$ is a derived attention weight,
not an additional stored parameter.
Since $0\le b-b_0<2^{-k}\le1$,
\[
 0\le\frac{b^2-b_0^2}{\sqrt d}
 \le\frac{(b-b_0)(2b_0+1)}{\sqrt d}\le1,
\]
and hence
\[
 64L^2\le\gamma\le64eL^2.
\]
Choose
\[
 e_s=af_1,\qquad e_\ell=f_3,\qquad
 e_\#=bf_2,\qquad w=af_1.
\]
Then $c_s=c_\ell=0$, $c_\#=b^2/\sqrt d$, and
\[
 \hat y(n,m)=\frac{a^2n}{n+m+\gamma}.
\]
Using $0\le a^2-\gamma\le2\sqrt\gamma+1$, we obtain
for every $(n,m)\in\Tset$,
\[
 |\hat y(n,m)-n|
 \le\frac{L(L+2\sqrt\gamma+1)}{\gamma}
 \le\frac1{64}+\frac14+\frac1{64L}
 <\frac12.
\]
For fixed $d$, the smallest admissible $k$ is
$O_d(\log\log L)$, while $b=O_d(\sqrt{\log L})$ and
$a=O(L)$.
All $4d$ stored scalar entries therefore lie on a common
binary fixed-point grid and require only $O_d(\log L)$
bits per scalar. At fixed dimension their total storage
is $O_d(\log L)$ bits. The vector norms are also
$O_d(L)$, consistent with the optimal norm order.
\end{proof}

This comparison permits different precisions for the two
models, but assumes explicit numerical storage.
It is not a lower bound on the length of a program that
generates the head: the quadratic class vectors in
Appendix~\ref{app:classification} have a short formula.
Stored-parameter precision also does not determine the
accuracy of numerical evaluation; arithmetic perturbations
must be controlled separately, as in the preceding
perturbation proposition.

\section{Consequences for counting range and dimension}
\label{app:consequences}

For fixed $d$, inversion of the proved strict-rounding bounds gives
$L_{\max}(B)=\Theta(B)$ for unnormalized softmax and learned pre-LN,
$\Theta(B^{2/3})$ for bare post-LN, and $\Theta(B^{4/3})$ for post-LN with a learned gain, with or without a learned bias.
The parameters may be chosen separately for each length. These bounds
do not determine the accuracy of a particular trained model at longer lengths.
The bare pre-LN cap is proved in Section~\ref{app:bare-pre}.
Linear attention without post-LN has one fixed exact construction for
all lengths once the budget accommodates it.

At fixed $B>0$ and as $L\to\infty$, the minimum dimensions are
\[
 d_{\min}(L;B)=
 \begin{cases}
 \Theta(L^4/B^4),&\text{learned pre-LN},\\
 \Theta(L^3/B^2),&\text{bare post-LN},\\
 \Theta(L^3/B^4),&\text{post-LN with learned gain (any bias)}.
 \end{cases}
\]
Both assertions follow by rearranging the dimension-uniform upper and
lower budget bounds; integer ceilings and the constraint $d\ge3$ affect
only constants in these asymptotic regimes.

For the exhibited accurate constructions, $\gamma=\Theta(L^2)$ gives
$\lambda_L=L/(L+\gamma)=\Theta(L^{-1})$, while
$\omega=L^{-3/2}$ gives $\theta_L=\arctan(L^{-1/2})=\Theta(L^{-1/2})$.
These explain the shrinking segment or arc used by the constructions;
they are not universal requirements on every possible representation.

\section{Sequential handoffs and parallel aggregation}
\label{app:necessity}

A literal running-counter implementation updates one counter after each token and passes the result to the next position. Transformer layers update all positions simultaneously from the previous layer, so a value newly computed at position $i$ cannot be used at position $i+1$ until the next layer. A chain of $L-1$ such dependent handoffs therefore needs at least $L-1$ layers. This is a limitation of that particular implementation, not a lower bound for all transformer counting algorithms. Parallel aggregation, as in Theorem~\ref{thm:nonorm}(a), avoids those sequential dependencies. This observation motivates studying a single attention sublayer without claiming that every counting method must use it.

\section{Empirical Results}
\label{empirical-appendix}

\subsection{Counting with Minimal Architecture}
\paragraph{Architecture} We use the minimal architecture defined in Section \ref{sec:setup}. Specifically, we use the softmax attention case, with no normalization for a clean, isolated study.

\paragraph{Training Setup}
We study $d\in\{4,16\}$ and $L\in\{250,500\}$.  For each length, the
dataset is the complete domain
$\mathcal{D}_L=\{(n,m): n,m\in\mathbb{Z}_{\geq 0},\ n+m\leq L\}$, with
targets $n$.  Thus, $|\mathcal{D}_{250}|=31{,}626$ and
$|\mathcal{D}_{500}|=125{,}751$. For each $L$, we randomly partition all admissible pairs into a fixed 80/20
train/test split, shared across the three model seeds.  The split is generated
deterministically with seed $L$; we do not stratify by the target count.

We train with full-batch Adam, minimizing mean squared error
$\frac{1}{|\mathcal{D}_{\mathrm{train}}|}
\sum_{(n,m)\in\mathcal{D}_{\mathrm{train}}}(\hat n-n)^2$.
We use PyTorch Adam defaults $(\beta_1,\beta_2,\epsilon)=(0.9,0.999,10^{-8})$,
zero weight decay, and no gradient clipping.  All parameters and computations
use float64.  For each setting, we train seeds $0,1,2$, initializing the three
token embeddings and scalar-readout vector independently as
$\mathcal{N}(0,I_d/d)$.

For the first 20M updates, the learning rate follows cosine annealing from
$2\times10^{-2}$ to $10^{-4}$.  The $L=250$ runs stop at 20M updates.  The
$L=500$ runs resume from the 20M checkpoint, retaining both model and Adam
optimizer states, and continue with a constant learning rate of $10^{-4}$ to
100M total updates.  We use no early stopping or test-set model selection;
metrics and resumable checkpoints are recorded every 50,000 updates.  The
three seed trajectories are run concurrently on a single NVIDIA A100 GPU.

\begin{figure}[t]
      \centering
      \includegraphics[width=0.85\linewidth]
      {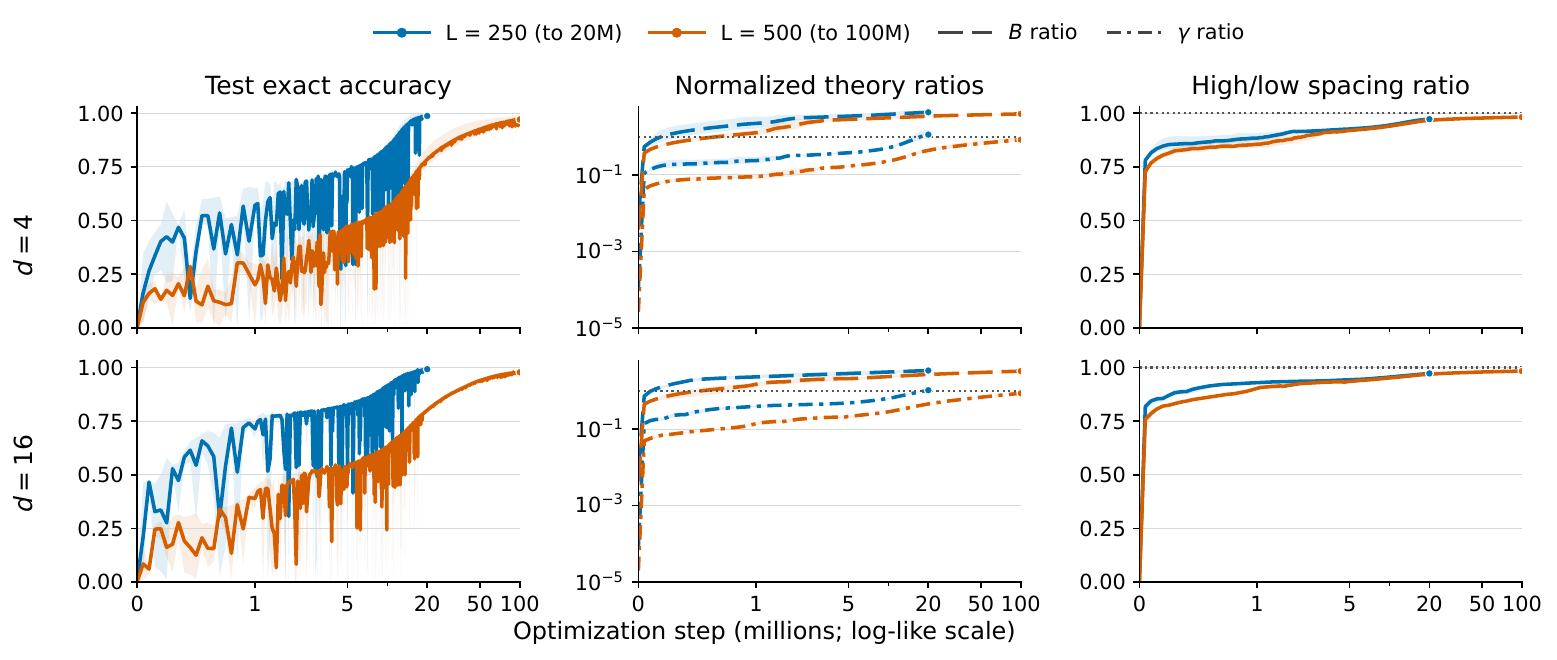}
      \caption{
        \textbf{Training and representation dynamics of the unnormalized softmax model.}
        Columns show held-out exact accuracy, theory-normalized parameter and
        gamma-score scales, and the high-/low-count adjacent-spacing ratio. Training is run for 3 seeds.
}
      \label{fig:theory-threshold-dynamics}
  \end{figure}

\paragraph{Results}
The \(L=250\) models achieve near-perfect held-out rounded accuracy, although only one of these six runs reports a full-domain worst error below $1/2$; the \(L=500\) runs are still improving at \(20\)M updates, and we therefore extend the run to $100$M steps for higher rounded accuracy. Table~\ref{tab:long-sweep-final-performance} reports the final values. We report \(B/B_{\mathrm{lower}}\) and \(\gamma/\gamma_{\mathrm{round}}\), where \(B_{\mathrm{lower}}=\tfrac12\sqrt{L^2-L}\) and \(\gamma_{\mathrm{round}}=\tfrac{L^2-1}{4}-L\). These are necessary lower bounds, not sufficient certificates of correct rounding (Lemma~\ref{app:scalar}); ratios above one alone do not establish success. On the \(m=0\) slice,
\[
h_n=h(n,0)=\frac{n e_s+\gamma e_\#}{n+\gamma},
\qquad
\Delta_n=\lVert h_{n+1}-h_n\rVert_2
=\frac{\gamma\lVert e_s-e_\#\rVert_2}
{(n+\gamma)(n+1+\gamma)}.
\]
The high-to-low gap ratio is
\(\langle\Delta_n\rangle_{\mathrm{last}\ 10\%}/
\langle\Delta_n\rangle_{\mathrm{first}\ 10\%}\). A ratio below one means that neighboring high-count representations are more crowded than neighboring low-count representations. Because $\gamma$ grows slowly, high-count representations are compressed early in training and become more evenly spaced as training progresses (Figure~\ref{fig:theory-threshold-dynamics}, right).

The models learn large parameter norms, but \(\gamma\) remains the limiting quantity, suggesting that gradient descent struggles to coordinate the parameters needed to reach the required saturation scale. We show aggregate results in Table \ref{tab:long-sweep-final-performance}. We also perform training with standard mini-batch gradient descent for comparison. As expected, we find that mini-batch gradient descent can also learn counting. Convergence takes more parameter updates in this comparison (Figure~\ref{fig:minibatch}); update counts alone do not compare examples processed or computational cost.

\begin{table}[t]
    \centering
    \caption{
        Final performance after 20M steps for $L=250$ and 100M steps for
        $L=500$. Entries report the mean and $[\min,\max]$ across three seeds.
        Exact accuracy is measured on the held-out test split. Worst error is
        evaluated over the complete admissible input domain; rounding-exact
        counting requires worst error $< 1/2$.
    }
    \label{tab:long-sweep-final-performance}
    \small
    \setlength{\tabcolsep}{4pt}
    \begin{tabular}{ccccc}
        \toprule
        Dimension $d$
        & Maximum count $L$
        & Steps
        & Test exact accuracy
        & Full-domain worst error \\
        \midrule
        4  & 250 & 20M  & $0.9878\ [0.9815,\,1.0000]$
                      & $0.650\ [0.444,\,0.757]$ \\
        16 & 250 & 20M  & $0.9921\ [0.9902,\,0.9951]$
                      & $0.668\ [0.638,\,0.687]$ \\
        \midrule
        4  & 500 & 100M & $0.9713\ [0.9598,\,0.9883]$
                      & $0.833\ [0.703,\,0.915]$ \\
        16 & 500 & 100M & $0.9778\ [0.9763,\,0.9800]$
                      & $0.793\ [0.778,\,0.803]$ \\
        \bottomrule
    \end{tabular}
\end{table}

\begin{figure}[t]
    \centering
    \includegraphics[width=0.8\linewidth]{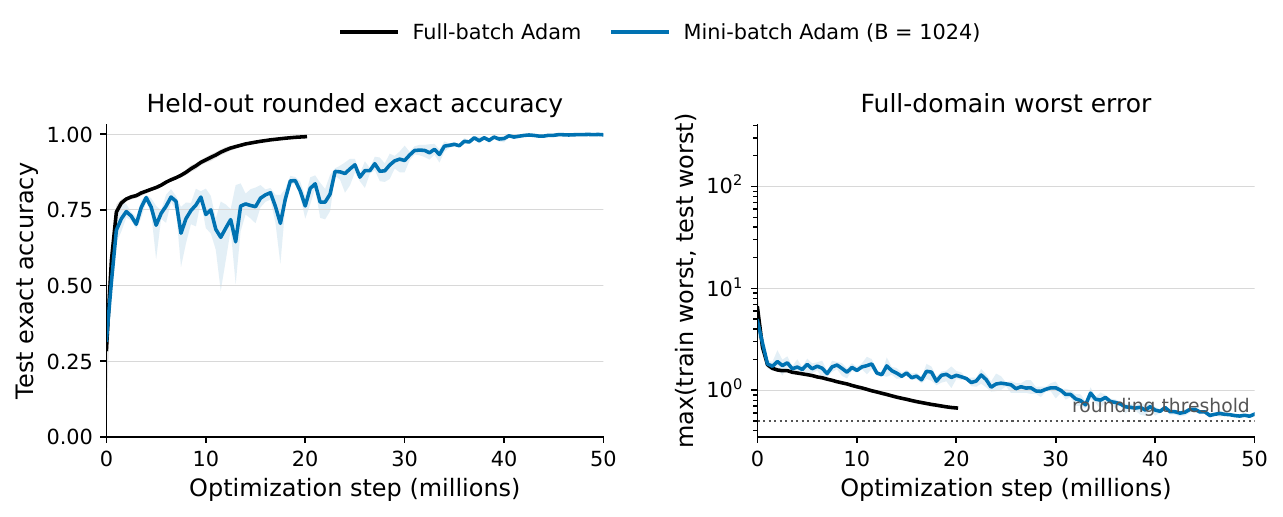}
    \caption{Training run for mini-batch Adam with batch size 1,024. Comparison with full-batch Adam, d=16 and L=250. The mini-batch run takes more updates in this comparison.}
    \label{fig:minibatch}
\end{figure}

\subsection{Post-LN experiments: protocol and details}
\label{app:postln-protocol}

\paragraph{Setup.}
One softmax head with identity projections, $d=16$, bare post-LN with $\varepsilon=0$, no residual connection. Scalar regression has no output bias; the classifier has one bias per class. Embeddings are initialized $\mathcal N(0,1/d)$ per coordinate and read-outs $0.1\,\mathcal N(0,1/d)$; the two heads share the embedding seed. For each $L$ and each space count $n$, the training set contains the letter counts $m=0$, $m=L-n$, and two more drawn at random ($3{,}998$ pairs at $L=1000$), weighted so every count contributes equally. A disjoint validation set, drawn the same way from the remaining pairs, is used only to choose the checkpoint; counts near $L$ have no remaining pairs ($n\ge997$ at $L=1000$) and are absent from it. All arithmetic is float64.

\paragraph{Optimization and full-domain check.}
L-BFGS with strong-Wolfe line search (10 iterations and at most 20 evaluations per call, history 20). Whenever the training, validation, and space-only sets are all correct, we enumerate every one of the $\binom{L+2}{2}$ inputs and require that each rounds correctly (regression) or gives the correct class a strictly positive winning gap (classification). This is a numerical check in float64, not a formal certificate. A run stops five logging intervals after first passing it, when its loss has not improved for 50 calls, or at 600\,s. Regression passes after 20--80 calls at every $L$, with worst error $0.0056$ at $L=1000$. The classifier passes after 80--120 calls at $L=100$, 740--2140 at $L=300$, and 3380 at $L=1000$ for one of three seeds, with smallest winning gap $9.5\times10^{-7}$; the other two seeds reached $99.99\%$ and $99.8\%$ of inputs at the time limit, the misses having gaps of order $-10^{-7}$.

\paragraph{What the panels show.}
Panel~(a) plots, for regression, the smallest read-out norm over seeds, and for classification the mean and range over seeds. The lower bound of Theorem~\ref{thm:post}(a) constrains $\|w\|$ alone and the construction attains the upper bound with arbitrarily small embeddings, so the read-out norm is the right quantity to compare with the band; trained embedding norms are larger and are recorded in the run files. A trained norm above the band is not a contradiction, since the band bounds the minimum budget, not every solution. Panels~(b) and~(c) use one run per $L$: the classifier seed with the largest minimum gap on the full domain, and regression seed~0. In panel~(c) each faint point is the minimum winning gap over all admissible $m$ for one count $n$; the thick curve is a rolling median of those minima over $L/25$ counts, for legibility only.

\paragraph{Logit multiplier.}
The classification loss is cross-entropy on $T\,(Wh+b)$ with $\log T$ a trainable scalar initialized at $\log 10$. Scaling $(W,b)$ by $a$ and $T$ by $1/a$ leaves the loss and the prediction unchanged, so the split between $T$ and the class norms is set by initialization and the optimizer, not by the objective; we report $T=3.8\times10^3$, $7.6\times10^4$, $3.5\times10^5$ (seed means) at $L=100,300,1000$, and $T\cdot\max_c\|w_c\|=2.5\times10^4$, $3.1\times10^5$, $8.0\times10^5$. In preliminary runs without the multiplier the classifier's loss stayed near $5$ nats at $L=1000$ and it reached $61\%$ of inputs at $L=100$ within the same wall-clock budget. The plausible reason is that with class norms of order one the softmax at unit scale spreads its mass over on the order of a hundred neighbouring counts, so the loss changes little when the arg max moves by one.

\paragraph{Budget frontier.}
Panel~(d) trains regression only. Each trainable vector $v$ is parameterized as $v\cdot\min(1,B/\|v\|)$; for the learned-gain variant the gain is capped too and the normalization bias is fixed at zero. For each variant and budget we walk the length grid $10,20,50,100,200,500,1000,2000,5000,10{,}000,20{,}000$ upward until the first failure, with up to three attempts of 60\,s per cell; when a solution that passed at a smaller budget or a shorter length exists, the first attempt starts from it, which is admissible since the caps are inequalities. Table~\ref{tab:frontier} lists the results. The dotted lines in panel~(d) are the continuous versions of these guarantees. Slopes in the main text are least-squares fits of $\log L_{\max}$ on $\log B$ over $B\le1000$. At $B=3000$ the constructions guarantee $L=5120$ without normalization and $L\approx8.7\times10^4$ with the learned gain, well above the $1000$ and $10{,}000$ attained, so at that budget the experiment underestimates known feasible capacity and we exclude the row from the fits. A possible cause is that the line search stalls at the norm cap, where the parameterization is not differentiable; we have not verified this.

\begin{table}[t]
\centering\small
\caption{Largest grid length passing the full-domain rounding check under a per-vector cap $B$, next to the largest integer $L$ each construction guarantees at $d=16$ (Appendix~\ref{app:one-knob} and Theorem~\ref{thm:post}). The next grid point above each attained value failed. Attained values are lower bounds on capacity; the $B=3000$ row is excluded from the slope fits.}
\label{tab:frontier}
\begin{tabular}{r rr rr rr}
\toprule
& \multicolumn{2}{c}{no LN} & \multicolumn{2}{c}{bare post-LN} & \multicolumn{2}{c}{post-LN, learned gain}\\
$B$ & attained & constr. & attained & constr. & attained & constr.\\
\midrule
10   & 10   & 15   & 20   & 11  & 100    & 43\\
30   & 50   & 50   & 50   & 24  & 500    & 186\\
100  & 100  & 169  & 100  & 54  & 2000   & 928\\
300  & 500  & 510  & 200  & 112 & 5000   & 4016\\
1000 & 1000 & 1705 & 500  & 251 & 10000  & 20000\\
3000 & 1000 & 5120 & 1000 & 524 & 10000  & 86535\\
\bottomrule
\end{tabular}

\end{table}

\paragraph{Further observations.}
The trained classifier's arc on the letter-free slice is $1.7$ rad, more than the quarter circle of the construction. On that slice the representation direction follows $A+nC$, so the angle swept from $n=0$ to $L$ is the angle between $A$ and $A+LC$, which approaches the angle between $A$ and $C$ only as $L\to\infty$; the construction sets that limiting angle to $\pi/2$, while the trained model chooses an obtuse one, which places the fastest-turning part of the arc mid-range (the hump in panel~b). The winning gaps in panel~(c) are nearly flat in $n$ although the angular gaps are not: the head uses longer class vectors where the arc is tight and shorter ones where it is loose, within a factor of ten, and the per-count gaps are approximated to within a small factor by $\|w_n\|\sqrt d\,\Delta_n^2/2$ computed from the learned geometry. On the regression side, the norm varies widely across seeds (at $L=300$: $1.3\times10^3$, $7.2\times10^5$, $7.6\times10^5$, all passing) because the squared error keeps falling as the arc shrinks, so the optimizer continues along the family of rounding-correct solutions until the stopping rule fires. In panel~(d), bare post-LN sits about twice above its construction at every budget, inside the range Theorem~\ref{thm:post}(a) allows; the learned-gain variant is above its construction up to $B=300$ and one grid step below it at $B=1000$, where the construction guarantees exactly $20{,}000$ but that cell failed.

\subsection{Representation Crowding in Large Language Models}
\label{representation-crowding-LLM}
Recent work reports representation crowding in LLMs \citep{bai2026count}.
To independently validate this qualitative phenomenon in a different model,
we analyze frozen residual-stream representations of
\texttt{Qwen/Qwen3-4B-Instruct-2507}. We use zero-shot prompts
rendered with the model's official chat template. Each prompt asks the model
to count a target token in a randomly permuted sequence of $L=256$
single-token items containing $n$ copies of the target and $L-n$ copies of a
distractor. The geometry results below use counts
$n\in\{1,\ldots,100\}$.

At each count and seed, there are 48 prompt configurations: the four ordered
target/distractor pairs \texttt{E/F}, \texttt{G/H}, \texttt{I/J}, and
\texttt{K/L}, crossed with twelve instruction forms. To introduce modest
linguistic variation, we prepend one of the tokens \texttt{Example},
\texttt{Trial}, \texttt{Sample}, or \texttt{Sequence}, and vary the counting
instruction among \texttt{Count}, \texttt{How many}, and \texttt{Number of}, e.g., \texttt{Sample. Count E in: E F E \ldots\string\n\#}. The
three seeds randomize item orders and the assignment of these configurations.
Thus there are $100\times48=4{,}800$ prompts per seed and 14,400 in total.

For each rendered chat prompt, let $T$ be its token length. We read the state
at token index $T-1$ (zero-indexed). We extract
\texttt{blocks.0.hook\_resid\_pre} and
\texttt{blocks.\{0,\ldots,35\}.hook\_resid\_post}, giving 37 residual
locations. The model forward pass is in bfloat16 and the saved activations are
float32. In particular, the final location is
\texttt{blocks.35.hook\_resid\_post}, the residual directly after the 36th
transformer block and \emph{before} the model's final RMSNorm; it is not the
Hugging Face \texttt{last\_hidden\_state}, which includes that final
normalization.

\subsubsection{Representation crowding}
To measure crowding without assuming a particular linear readout, for each
layer $\ell$, token pair $p$, and count $n$, we form the centroid
$\bar{h}_{\ell,p}(n)$ across surface forms. We then measure the Euclidean gap
between adjacent counts,
\begin{equation}
    \Delta_{\ell,p}(n)
    = \left\lVert
        \bar{h}_{\ell,p}(n+1)-\bar{h}_{\ell,p}(n)
      \right\rVert_2,
\end{equation}
and report the ratio of the mean gap among counts $91$--$100$ to that among
counts $1$--$10$. Each window contains ten centroids and nine within-window gaps ($n=91,\ldots,99$ and $n=1,\ldots,9$, respectively). A ratio below one means that adjacent high-count
representations are closer together than adjacent low-count representations.

Figure~\ref{fig:qwen3-crowding-rank} shows that the high/low gap ratio decreases
through the network and remains well below one in later layers. At the final
residual location, the ratio is $0.448$ on average across seeds and token pairs
(seed range $0.439$--$0.454$). Adjacent counts in the $91$--$100$ window have less than half the
separation of those in the $1$--$10$ window.

We additionally measure whether this compression is confined to a single
direction. For the ten centroids in a count window, let $\mu_i$ be the
eigenvalues of their centered Gram matrix and
$q_i=\mu_i/\sum_j\mu_j$. When $\sum_j\mu_j>0$, we report the entropy effective rank
\begin{equation}
    r_{\mathrm{eff}}
    = \exp\!\left(-\sum_{i:q_i>0}q_i\log q_i\right).
\end{equation}
If the centroids coincide, this normalization is undefined; such locations must be omitted rather than assigned a rank. Ten centered centroids have rank at most nine.
At the final layer, $r_{\mathrm{eff}}$ is $1.300$ for counts $1$--$10$ and
$6.764$ for counts $91$--$100$, averaged across seeds and token pairs. The
high-count centroid trajectory is more curved or dispersed across directions,
whereas the low-count trajectory is nearly one-dimensional. The coexistence of shrinking gaps and a curved count trajectory is
qualitatively consistent with the crowding mechanism in our minimal models,
but an effective rank of $6.764$ is not evidence for a one-dimensional or planar arc. This is not a
test of the same mechanism: Qwen3 has many attention heads, MLPs,
residual connections, positional embeddings, and normalization layers. In addition, these prompts fix $n+m=256$, whereas the theoretical spacing plots use $m=0$. With equal space and letter scores, our unnormalized model has equally spaced representations along a fixed-total-length slice, so the observed Qwen crowding does not follow directly from the softmax segment formula.

\begin{figure}[t]
    \centering
    \includegraphics[width=0.5\linewidth]{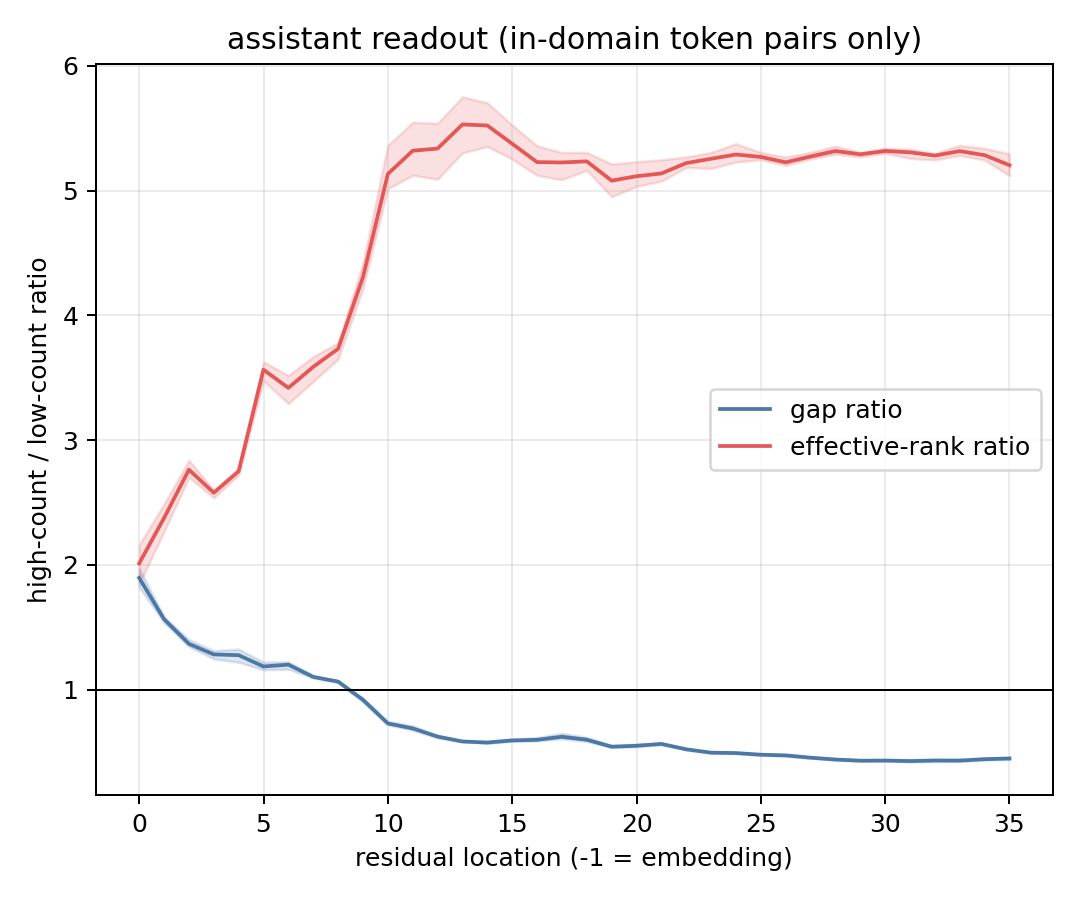}
    \caption{Layerwise geometry of the zero-shot assistant-boundary
    representation at $L=256$. Blue shows the mean adjacent-centroid gap for
    counts $91$--$100$ divided by that for counts $1$--$10$; red shows the
    corresponding effective-rank ratio. Lines are means over three seeds after
    averaging four fixed token pairs, and shaded regions show the seed range.
    The horizontal line denotes equality between the high- and low-count
    windows.}
    \label{fig:qwen3-crowding-rank}
\end{figure}

Figure~\ref{fig:qwen3-count-pca} provides a complementary visualization of the
count trajectory. The path becomes increasingly curved through the network,
and the high-count points become visibly denser in the later layers.

\begin{figure}[t]
    \centering
    \includegraphics[width=\linewidth]{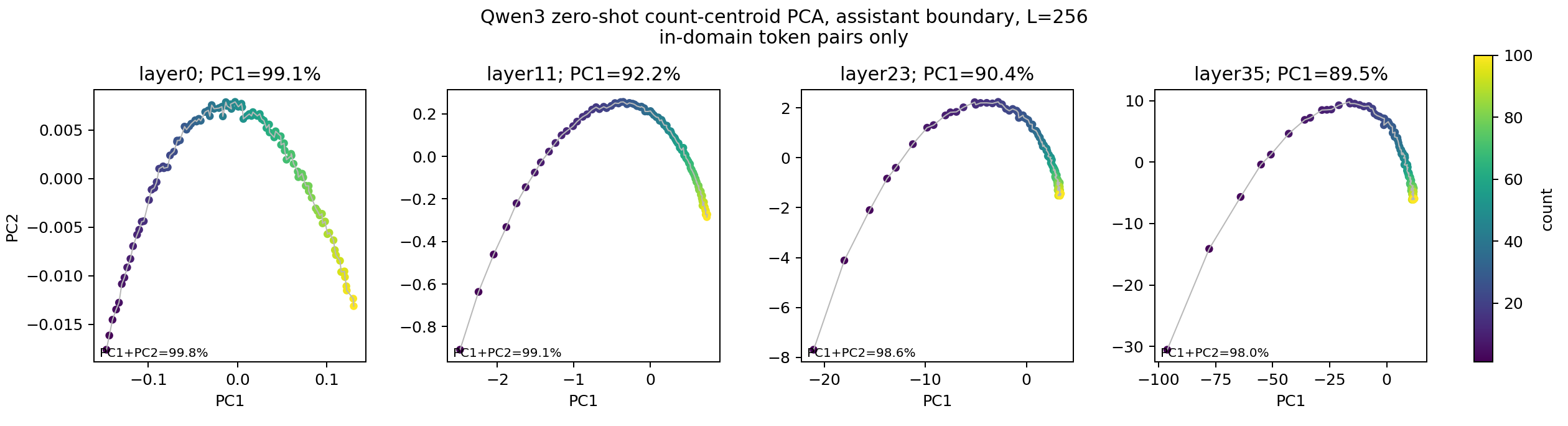}
    \caption{PCA of zero-shot assistant-boundary count centroids for
    $n=1,\ldots,100$ at four residual locations. Centroids are first averaged
    across prompt configurations and the three seeds, and color denotes count.
    Each panel has its own PCA basis and axis scale, so only within-panel
    geometry should be interpreted.}
    \label{fig:qwen3-count-pca}
\end{figure}

\subsubsection{Probing}

We next test how easily the count can be read out from the assistant-boundary
representation. At each residual location $\ell$, we fit a ridge-regularized
linear probe
\begin{align}
    \widehat{n}_{\ell,\lambda}(x)
    &=w_{\ell,\lambda}^{\top}h_{\ell}(x)+b_{\ell,\lambda},\\
    (w_{\ell,\lambda},b_{\ell,\lambda})
    &\in\operatorname*{arg\,min}_{w,b}
      \left\{\frac{1}{|\mathcal{D}_{\mathrm{fit}}|}
      \sum_{(x,n)\in\mathcal{D}_{\mathrm{fit}}}
      \bigl(w^{\top}h_{\ell}(x)+b-n\bigr)^2
      +\lambda\|w\|_2^2\right\}.
\end{align}
where the intercept is not regularized.

For each seed, we fit candidate probes at each of the 37 residual locations.
Because residual representations at different depths can have different
activation scales, we select regularization on a scale-normalized grid rather
than using a common absolute penalty. Specifically, if $X_c$ is the centered
fitting representation matrix with $N$ examples, we parameterize the penalty
as
\[
    \lambda
    = \lambda_{\mathrm{rel}}\,
      \lambda_{\max}\!\left(\frac{X_c^\top X_c}{N}\right),
\]
where $\lambda_{\max}$ denotes the largest empirical covariance eigenvalue.
We select $\lambda_{\mathrm{rel}}$ from the 34-value grid
\[
    \{0\}\cup\left\{10^{-12+16j/32}:j=0,\ldots,32\right\}.
\]
This normalization makes a given grid value correspond to a comparable
regularization scale across residual locations. Centering,
eigendecompositions, and ridge fits are performed in float64; the intercept is
fit separately and is not regularized. The absolute penalty is recomputed when
the selected probe is refit on the larger development set.

For each seed and count, the 48 prompt configurations are divided into 24
fitting, 8 validation, and 16 IID evaluation configurations. The split unit is
a complete configuration---the joint choice of an ordered
target/distractor pair and an instruction prefix---rather than merely a new
item ordering. We jointly select the residual location and
$\lambda_{\mathrm{rel}}$ by rounded exact accuracy on the validation
configurations, refit the selected probe on all $24+8=32$ development
configurations, and evaluate it once on the 16 held-out configurations.
Selection and refitting use the full count domain $n=0,\ldots,256$;
Figure~\ref{fig:qwen3-probe-by-count} displays the
$n=1,\ldots,100$ evaluation slice. We define
$\operatorname{round}(z)=\lfloor z+1/2\rfloor$, so rounded exact accuracy is
$\one\{\operatorname{round}(\widehat n)=n\}$.

For all three seeds, validation selects the output of the first transformer
block and relative ridge coefficient $10^{-7}$. Its overall IID rounded
accuracy on the full $L=256$ task is $0.666$ (seed range $0.634$--$0.687$).
Accuracy is highest in the low-count
window: it is $0.765$ for counts $1$--$10$, then broadly declines with count,
reaching $0.623$ for counts $91$--$100$ (Figure~\ref{fig:qwen3-probe-by-count}).
The $14.2$ percentage-point difference is consistent with, but does not establish a causal link to, the geometric observation
that adjacent high-count representations are more crowded.

\begin{figure}[t]
    \centering
    \includegraphics[width=0.5\linewidth]{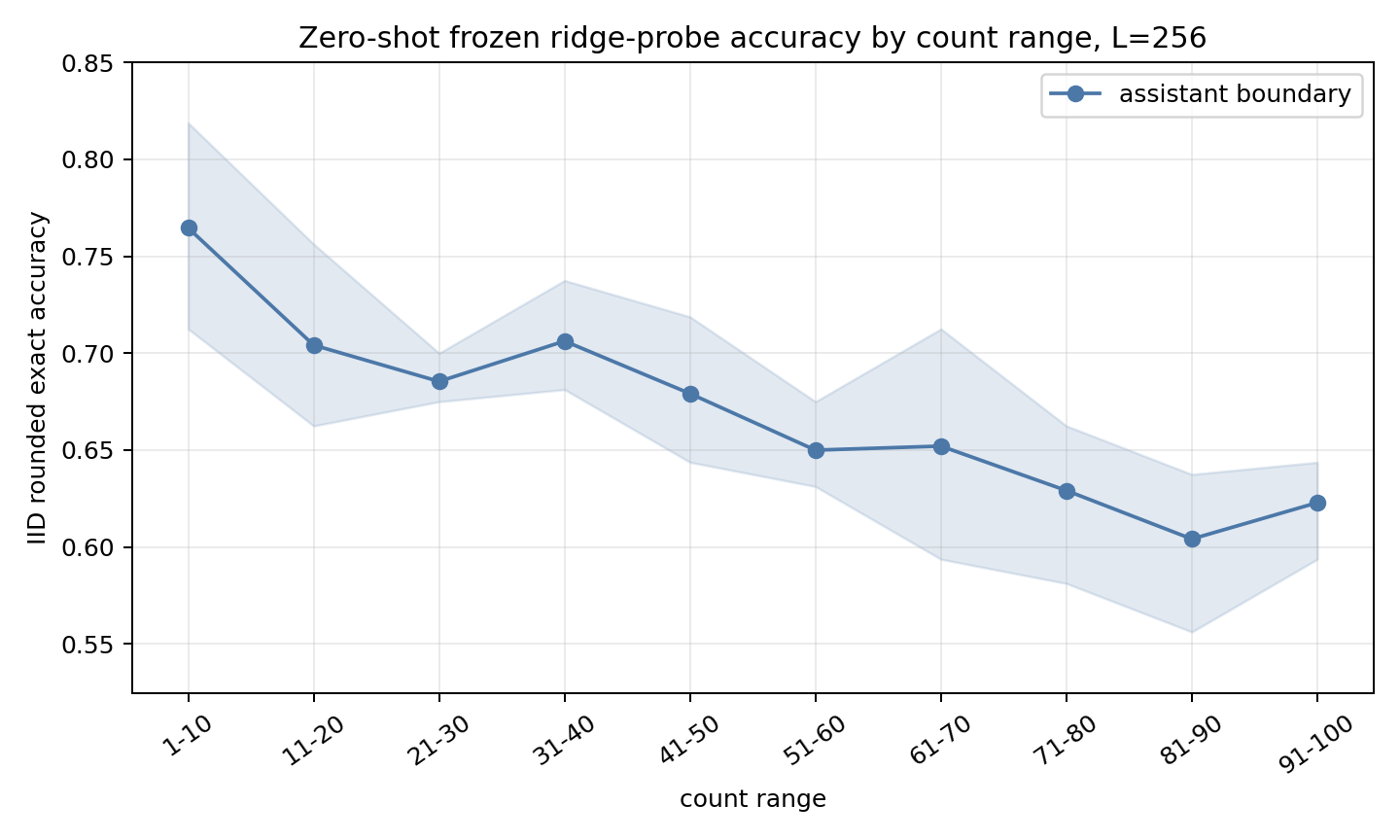}
    \caption{IID rounded exact accuracy of the validation-selected zero-shot
    ridge probe, binned into ten-count ranges. Points show the mean across three seeds and shading shows
    the seed range.}
    \label{fig:qwen3-probe-by-count}
\end{figure}

\paragraph{Regularization--norm tradeoff.}
The crowding results suggest that resolving neighboring counts may require a
large readout norm. To test this directly, we fix the representation to the
assistant boundary after the first transformer block and sweep the ridge
coefficient further while refitting the probe on all 32 development variants. As
regularization is weakened, the probe norm increases and the attainable
held-out accuracy initially improves (Figure~\ref{fig:qwen3-ridge-norm}). At
$L=256$, the smallest norms among the sampled probes reaching $50\%$ and $60\%$
accuracy average $1{,}202$ and $1{,}542$, respectively; reaching $65\%$
has a smallest sampled norm averaging $1{,}711$ among the two seeds that attain this
threshold. The highest observed accuracy along the path is $0.671$ at mean
norm $2{,}341$.

This exposes a practical tension. Regularization is useful for probes of real
LLM representations because the high-dimensional residual stream contains
many task-irrelevant directions and finite-sample variation. However, overly
strong regularization prevents the probe from acquiring the large norm needed
to amplify the small differences between crowded count representations. The
best observed solution along this sampled path is weakly, but not completely, regularized: the
unregularized probe has a still larger mean norm of $3{,}440$ but lower accuracy
of $0.633$, consistent with overfitting. These native-coordinate readout norms
provide evidence for a norm--accuracy tradeoff along the sampled ridge path, not a lower bound over all probes. The best held-out accuracy on this path is exploratory; it is separate from the validation-selected evaluation above.

\begin{figure}[t]
    \centering
    \includegraphics[width=0.62\linewidth]{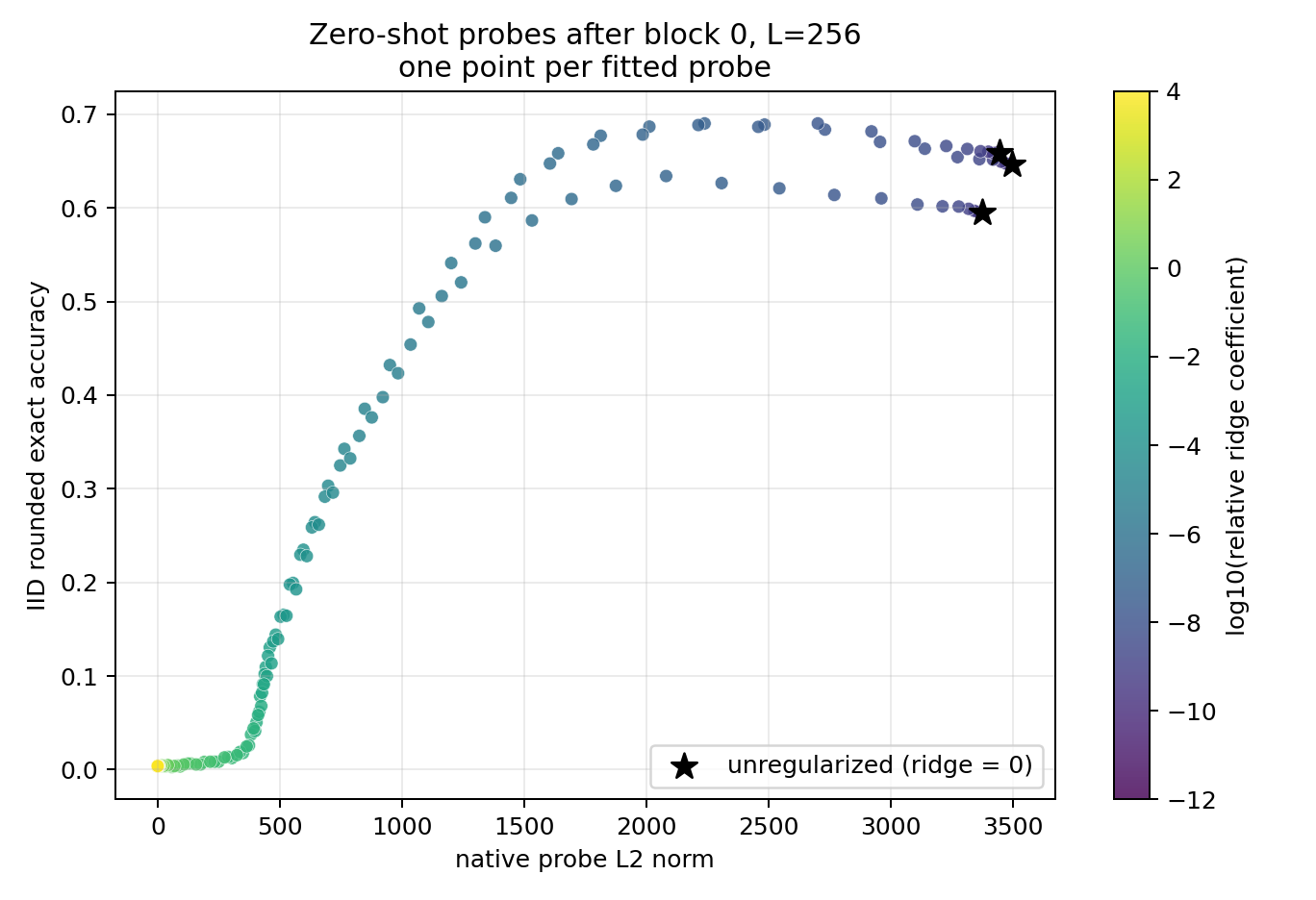}
    \caption{Regularization--norm tradeoff for zero-shot probes at the
    assistant boundary after the first transformer block. Each colored point
    is one fitted probe: horizontal position gives its native $\ell_2$ norm,
    vertical position gives held-out rounded accuracy, and color gives the
    relative ridge coefficient. Black stars denote the three unregularized
    probes.}
    \label{fig:qwen3-ridge-norm}
\end{figure}

\end{document}